\documentclass[english,12pt]{article}
\usepackage[T1]{fontenc}
\usepackage[latin9]{inputenc}
\usepackage{color}
\usepackage{textcomp}
\usepackage{amsmath}
\usepackage{amsthm}
\usepackage{amssymb}
\usepackage{graphicx}
\usepackage{wasysym}
\usepackage{hyperref}
\usepackage{enumitem}
\usepackage[normalem]{ulem}

\newcommand{\bm}{\boldsymbol}
\makeatletter
\usepackage{xcolor,cancel}

\newtheorem{theorem}{Theorem}[section]
\newtheorem{lemma}[theorem]{Lemma}

\newtheorem{corollary}[theorem]{Corollary}

\newtheorem*{theorem*}{Theorem}
\newtheorem*{lemme*}{Lemme}

\newtheorem*{proposition*}{Proposition}

\newtheorem{assumption}[theorem]{Assumption}
\newcommand{\eps}{\varepsilon}
\newcommand{\E}{\mathbb{E}}

\newcommand{\btheta}{\boldsymbol{\theta}}
\theoremstyle{definition}

\theoremstyle{remark}
\newtheorem{remark}[theorem]{Remark}

\author{Inbar Seroussi\footnote{Department of Mathematics, Weizmann Institute of Science, Rehovot, Israel}\;\;\;\;\;\;\;
Ofer Zeitouni\footnote{
Department of Mathematics, Weizmann Institute of Science, Rehovot, Israel}}
\date{}

\begin{document}
\title{Lower Bounds on the Generalization Error of Nonlinear Learning Models}
\maketitle
\begin{abstract}
 We study in this paper lower bounds for the generalization error of
 models derived from
 multi-layer neural networks, in the regime where the
 size of the layers is commensurate with the number of samples in the
 training data. We derive explicit
 generalization lower bounds for
 general biased estimators, in the cases of two-layered networks. For linear activation function, the bound is asymptotically
 tight. In the nonlinear case,
 we provide a comparison of our bounds with an empirical study
 of the stochastic gradient descent algorithm. In addition, we derive
 %compare to similar 
 bounds for unbiased estimators, which show that the latter have unacceptable
 performance for truly nonlinear networks.
 %. Our results show that they have unacceptable
 %performance
 %for such nonlinear networks in this regime. 
 The analysis uses
 elements from the theory of large random matrices.
\end{abstract}
\section{Introduction}

The empirical success of deep learning is notable in a vast array of applications such as image recognition \cite{krizhevsky2012imagenet},
speech recognition \cite{hinton2012deep} and other applications \cite{goodfellow2016deep,schmidhuber2015deep}.
In spite of many recent advances, this empirical success continues to outpace the development of a concrete theoretical
understanding of the optimization process; in particular, the answer to the question of
 how and when deep learning algorithms generalize well is still open.
One of the difficulties stems from the nonlinear and complex structure of these networks which are highly
non-convex functions, often with millions of parameters. Moreover, the design of such networks for specific applications is currently mainly done in practice by trial and error.

In this paper, we derive  Cram\'er-Rao (CR) type lower bounds \cite{rao1947minimum,cramer1946contribution} on the generalization error of these networks, when the training data is noisy. These bounds may provide insights in the process of designing and evaluating learning algorithms. Our analysis is motivated by the use of CR bounds in engineering as a benchmark for performance evaluations, see e.g.  \cite{tichavsky1998posterior,eldar2008rethinking}  and references therein,
and as a guideline in the process of improving the experimental design.
To calculate the CR bounds we combine tools from statistical estimation theory and random matrix theory. Our analysis is flexible and can be generalized to other learning tasks and architectures.

Our main findings are as follows.
We provide  lower bounds for general estimators, and evaluate them
using random matrix theory 
%with
%no assumption on the estimator bias. 
%We calculate these bounds 
for linear models and for feed-forward neural networks.
%using random matrix theory.
%The main matrix that appears when calculating these bounds is the Fisher information matrix.
We show that the bounds are tight for two layers neural network 
with linear activation function. As a comparison, we provide also a lower bound on the generalization error for  unbiased estimators in nonlinear models,
based on the classical CR bound; that bound is
%. The error in this case is bounded by 
the ratio of expected \textit{rank} of the Fisher information matrix to the number of samples, times the variance of the noise. 
We show that the expected rank of the Fisher information matrix in 
high dimension is large, for standard nonlinear network architectures, thereby
confirming that 
%(Motivated by the ``double descent'' phenomena in high dimension \cite{nakkiran2019deep,mei2019generalization,belkin2019reconciling,belkin2019two}, the expected rank may be
%viewed as the interpolation threshold for nonlinear models, if unbiased
%estimators are used.)
successful learning algorithms for nonlinear
networks in high dimension need to be biased.

%We emphasize that we consider the case of noisy training data, see
%\eqref{eq:model} below. It would be very interesting to derive a-priori lower bounds on the performance of estimators
%in the (somewhat unrealistic) setup of
%noiseless training data. However, our methods completely break down in that case.

\subsection{Related literature}
The connection between random matrix theory and deep learning is not new. It first appeared  in the study of neural networks at initialization,
for deterministic \cite{louart2018random,fan2020spectra} and  random \cite{pennington2017nonlinear,benigni2019eigenvalue} data. In these works, the matrix of interest is the conjugated kernel, i.e., the output of the last hidden layer at different data points transposed with itself.
The spectrum of this kernel can then be used in the evaluation
 of the training error, as well as of the generalization error in the random features model in which only the last layer weights are learned
 \cite{louart2018random,pennington2017nonlinear,benigni2019eigenvalue,mei2019mean}.
Random matrix theory is also used in the context of kernel learning. It is shown in \cite{jacot2018neural} that in the limit of large number of parameters and finite training samples,
 deep networks can be viewed as a kernel learning problem. Here, one of the matrices of interest is the Neural Tangent Kernel \cite{jacot2018neural}, whose
 spectrum in the  ``linear width limit'' when the number of samples is proportional to the hidden layers or features vector size, is calculated in \cite{fan2020spectra}.
%This kernel is also important for the understanding the optimization landscape.

Another matrix of interest is the Fisher information matrix which  measures the amount of information about unknown parameters of the true model distribution that the training samples carry. This matrix appears  in the natural gradient algorithm, since the latter
is  the steepest descent algorithm induced by the Fisher geometry metric. This algorithm has the advantage of being invariant under re-parametrization \cite{amari1998natural}.
The Fisher matrix is also used to define a notion of complexity using  the Fisher-Rao norm, en route to providing  an upper bound on the generalization error of deep network with
the Relu activation function \cite{liang2019fisher}. The spectrum of the Fisher information  matrix at initialization for one hidden layer is calculated in \cite{pennington2018spectrum}. The Fisher matrix for deep neural network in the mean field limit is studied in \cite{karakida2018universal}.

There are several existing generalization upper bounds. Most of these bounds aim to estimate the capacity of the model by offering new measures of complexity excluding knowledge about the true prior of the model's parameters. This idea is used in order to bound the generalization error from above, examples for such bounds are the PAC-Bayes bounds \cite{mcallester1999pac,neyshabur2017pac,dziugaite2017computing}, VC dimension \cite{vapnik2015uniform} parameters norms \cite{liang2019fisher,bartlett2017spectrally,neyshabur2017exploring}. For the empirical evaluation of some of these generalization bounds see  \cite{jiang2019fantastic}. These approaches suggest that modern network architectures have very large capacity.
Recently, lower bounds on the generalization error of linear regression models in the  overparametrized regime are derived
in \cite{bartlett2020benign,dicker2016ridge,hastie2019surprises,dobriban2018high} for specific estimators and for any interpolating estimator in \cite{muthukumar2020harmless}.

\subsection{Organization}
The remainder of this paper is organized as follows. In Sec. \ref{sec: general settings}, we present the model and assumptions. In Sec. \ref{sec: main results}, we present our main analytical results.
%about the bounds for general estimators and for unbiased estimator. 
In Sec. \ref{sec:numerics}, we compare our bounds to some known estimators.
In Sec. \ref{sec: proofs} we provide the proof of the main theorems.  For the reader's convenience,
a review of the CR bounds is provided in Appendix \ref{Sec:CRB}.

\subsection{Notation} Throughout,  boldface lowercase letters denote  (column) vectors. $\boldsymbol{x}^T$ denotes the transpose of a vector
$\boldsymbol{x}$.
Uppercase letters denote matrices. For two vectors $\boldsymbol{v},\boldsymbol{w}$ of the same length, $\boldsymbol{v}\circ\boldsymbol{w}$ denotes the vector whose $i$th entry is $\boldsymbol{v}_i\cdot \boldsymbol{w}_i$.  For reals $a,b$, we write $a\wedge b=\min(a,b)$. We denote by $\otimes$ the Kronecker product.

 For a random vector $\boldsymbol{\theta}$, the statement
$\boldsymbol{\theta}\sim p(\boldsymbol{\theta})$ means that $\boldsymbol{\theta}$ is distributed according to the law $p(\boldsymbol{\theta})$. When a law has density with
respect to Lebesgue measure on Euclidean space, we continue to use $p$ for the density; no confusion should arise from this.

We use the standard $O$ notation. Thus, sequences $a=a(d)$ and $b=b(d)$ satisfy $a=O(b)$ if there exists a constant $C$ so that $a(d)\leq C b(d)$ for all $d$.
Similarly, $a=o(b)$ if $\lim_{d\to\infty} |a/b|=0$.

We write $\mathbb{E}$ for expectation. When we want to emphasize over which variables expectation is taken, we
often write e.g. $\mathbb{E}_{x,y}$. When a conditional law is involved, we write e.g. $\mathbb{E}_{x,y|\boldsymbol{\theta}}$. Thus, in the last expression, the expectation is taken with respect to the law
$p(x,y|\boldsymbol{\theta})$.

For an $N\times N$  matrix $A$  with eigenvalues $\lambda_i$, we use $\rho_A=N^{-1}\sum_{i=1}^N \delta_{\lambda_i}$ to denote the empirical measure of eigenvalues of $A$.
We use $\rho_{\gamma}$ to denote the Marchenko--Pastur distribution with parameter $\gamma\in (0,\infty)$, i.e., with  $\lambda_{\pm}=(1\pm\sqrt{\gamma})^{2}$,
	\begin{equation}
\label{eq-PM}
	d\rho_{\gamma}(s)={\bf 1}_{s\in[\lambda_{-},\lambda_{+}]}\sqrt{(\lambda_{+}-s)(s-\lambda_{-})}/(2\pi\gamma s)ds+(1-{\gamma}^{-1})_{+}\delta_{0}(s).
	\end{equation}

	\noindent
	{\bf Acknowledgment} We thank Yonina Eldar and Ohad Shamir for
	useful discussions at the initiation of this project.
	 This work was partially supported by the Israel Science Foundation
	grant \# 421/20 and by the European Research Council (ERC) under the European Union's Horizon 2020 research and innovation program (grant agreement No. 692452). This paper was presented at TOPML, April 20-21 2021, as a short presentation.
\section{Model and problem statement \label{sec: general settings}}

We begin by setting a  general framework for learning problems. We then specialize it to the feed-forward networks that are studied in this paper.
%, which is treated in this paper.
\subsection{A general learning model}
\label{sec-genlearning}
We are given
 a set of $M$ training samples $\{\boldsymbol{z}^{(i)}\}_{i=1}^{M}$, such that $\boldsymbol{z}^{(i)}=(\boldsymbol{x}^{(i)};{y}^{(i)})\in\mathbb{R}^{d}\times \mathbb{R}$ are independently drawn from a distribution $p(\boldsymbol{z}^{(i)}|\boldsymbol{\theta})$, parameterized by a vector $\boldsymbol{\theta}\in\mathbb{R}^P$.
The parameters $\boldsymbol{\theta}$ are assumed random,  such that $\boldsymbol{\theta}\sim p(\boldsymbol{\theta})$, with $p(\boldsymbol{\theta})$  the  \textit{prior distribution}.
We write $X\in\mathbb{R}^{d\times M}$ for the matrix whose columns are $\boldsymbol{x}^{(i)}$,  $\boldsymbol{y}\in\mathbb{R}^{M}$ for the vector  whose
entries are $y^{(i)}$,  and set  $Z=\begin{pmatrix}X\\ \boldsymbol{y}^T \end{pmatrix}\in\mathbb{R}^{(d+1)\times M}$.

We specialize the model to the situation where
\begin{equation}
{y}^{(i)}=f_{\boldsymbol{\theta}}(\boldsymbol{x}^{(i)})+{\varepsilon}^{(i)}
\label{eq:model}
\end{equation}
where $f_{\boldsymbol{\theta}}:\mathbb{R}^d  \rightarrow \mathbb{R}$ is some representation of the real world,  ${\varepsilon}^{(i)}\sim\mathcal{N}(0,\sigma_{{\varepsilon}}^{2})$ are iid, and
$\boldsymbol{x}^{(i)}\sim p(\boldsymbol{x})$; that is, the independent training samples $\boldsymbol{x}^{(i)}$ have a distribution independent of $\boldsymbol{\theta}$, while the
$y^{(i)}$s do depend on $\boldsymbol{\theta}$ and observation noise $\varepsilon^{(i)}$.

The estimation task at hand is to predict the output of the network $\tilde{{y}}=f_{\boldsymbol{\theta}}(\tilde{\boldsymbol{x}})$ on a new sample $\boldsymbol{\tilde{x}}\sim p(\boldsymbol{{x}})$.
Formally, we seek a (measurable) \textit{estimator}
$\hat{{y}}(\tilde{\boldsymbol{x}},Z)$. The performance of  an estimator is given in terms of a \textit{loss function} $\ell:\mathbb{R}\times \mathbb{R}\to \mathbb{R}$, which we take to be the quadratic loss $\ell(y,y')=(y-y')^2$; that is we consider  the generalization error (population risk) $\mathbb{E}_{Z,\tilde{\boldsymbol{x}}}[ \ell(\tilde{{y}};\hat{{y}})]$.
One seeks  of course to minimize the population risk, and we will derive limitations on how small a risk can one achieve.

It is worthwhile to emphasize that  we do not impose any constraint on the estimator $\hat y$ (except for the technical measurability condition). In particular, we do \textit{not}
require that $\hat y=f_{\hat {\boldsymbol{\theta}}}(\tilde{ \boldsymbol{x}})$ for some $\hat{\boldsymbol{\theta}}$, where $\hat{\boldsymbol{\theta}}$ is a measurable function of
$Z$; various learning algorithms may of course use such estimators, and our lower bounds on the mean square loss will apply to them. However, the derivation of the bounds does not postulate
such a structure for the estimator.

\iffalse
It is convenient to decompose the generalization error as
\begin{equation}
\mathbb{E}_{\tilde{\boldsymbol{x}},Z,\boldsymbol{\theta}}\left(\tilde{y}-\hat{y}\right)^{2}=\underset{\mathrm{Variance}}{\underbrace{\mathbb{E}_{\tilde{\boldsymbol{x}},\boldsymbol{\theta}}\mathrm{Var}_{Z|\boldsymbol{\theta},\tilde{\boldsymbol{x}}} \left[\hat{y}\right]}}+\underset{\mathrm{Bias}}{\underbrace{\mathbb{E}_{\tilde{\boldsymbol{x}},\boldsymbol{\theta}}\left[b_{\boldsymbol{\theta}}(\tilde{\boldsymbol{x}})^{2}\right]}},
\label{eq: bias variance decompossion}
\end{equation}
where the bias is denoted by $b_{\boldsymbol{\theta}}(\tilde{\boldsymbol{x}})=f_{\boldsymbol{\theta}}(\tilde{\boldsymbol{x}})-\mathbb{E}_{Z|\boldsymbol{\theta},\tilde{\boldsymbol{x}}}\hat{y}(\tilde{\boldsymbol{x}},Z)$. \textit{Unbiased estimators} are those which satisfy $b_{\boldsymbol{\theta}}(\tilde{\boldsymbol{x}})=0$, almost surely.
Note that the minimum mean square error (MMSE) estimator
$\hat{y}_{\mathrm{opt}}=\mathbb{E}\left[\tilde{y}|\tilde{\boldsymbol{x}},Z\right]$ is typically not unbiased.
\fi
We often make the following assumption on the structure of the data.
\begin{assumption}
\label{subsec:assumptions}
	The $\boldsymbol{x}^{(i)}$, $i=1,\ldots,M$, are iid centered  Gaussian vectors with
	covariance matrix $\mathbb{E}\boldsymbol{x}^{(i)} (\boldsymbol{x}^{(i)})^{T} = \sigma^2_xI_d$.
\end{assumption}
This is a standard assumption, which simplifies the calculations. We believe it can be relaxed significantly (i.i.d. entries with sub-Gaussian tails should
present no difficulty), however this would involve a significant amount of technical work that we chose to avoid.

\subsection{The feed-forward setup}
A feed-forward network with $L$ layers and parameters $W^{(l)}\in \mathbb{R}^{N_l\times N_{l-1}}$, $ l=1,\ldots,L$, see Figure \ref{fig-NN},  is one where
\begin{equation}
f_{\boldsymbol{\theta}}(\boldsymbol{x})=\frac{1}{\sqrt{N_{L-1}}}W^{(L)}\sigma(\boldsymbol{q}^{L-1}), %+\boldsymbol{b}^L,
\label{eq: deep model}
\end{equation}
with  the $q^l$s, $l=1,\ldots,L$,  defined recursively as $\boldsymbol{q}^{1}=W^{(1)}\boldsymbol{x}/\sqrt{d}$
and  $\boldsymbol{q}^{l}=W^{(l)}\sigma(\boldsymbol{q}^{l-1})/\sqrt{N_{l-1}}$ for $l\geq 2$; here, $\sigma:\mathbb{R}\to\mathbb{R}$ is a continuous  function,
$\sigma(\boldsymbol{x})_i=\sigma(\boldsymbol{x}_i)$, and $W^{(l)}$ are  the \textit{weights} of the network. Since our output is a scalar, we always have $N_L=1$.
The vector $\boldsymbol{\theta}
\in \mathbb{R}^P$ is then taken as the collection
of weights $W^{(i)}$, $i=1,\ldots,L$, with $P=\sum_{l=1}^{L}N_{l}N_{l-1}$.
\begin{figure}[h!]
	\begin{centering}
		\includegraphics[trim={3cm 0cm 1cm 0cm},scale=0.3]{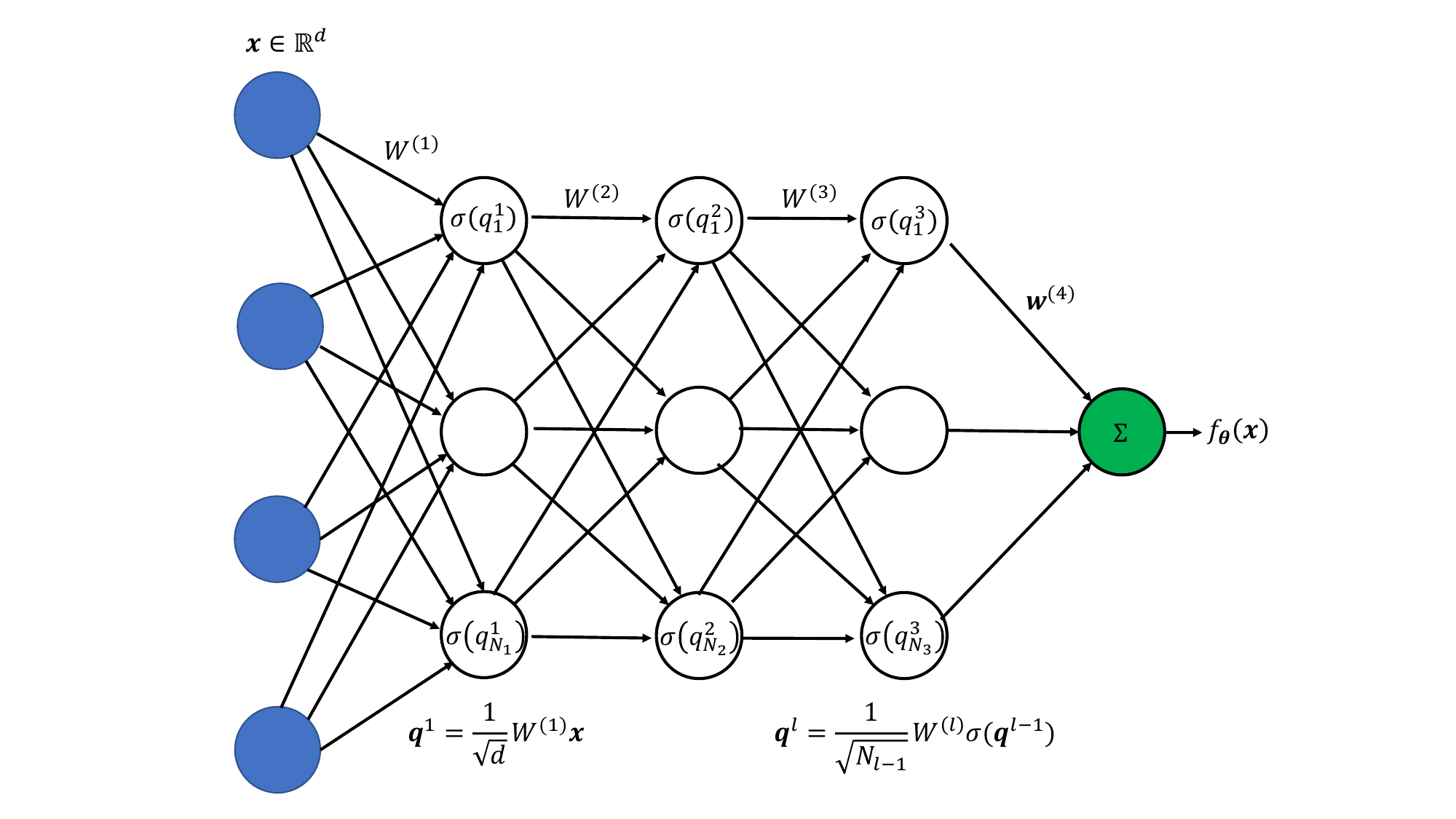}
		\par\end{centering}
	\caption{\label{fig-NN}Feed-forward network architecture, with $L=4$.}
\end{figure}

We often make the following assumption on the structure of feed-forward networks.
%Recall that in our setup, those are really assumptions on the
%structure of the data.
\begin{assumption}
\label{subsec:assumptionsiFF1}
	$W_{ij}^{(l)}\sim \mathcal{N}(0, \alpha_l^{-1})$, independent, with  ${\alpha_l}>0$, $\alpha_l=\alpha$ for $1 \leq l<L$.
\end{assumption}
 Assumption \ref{subsec:assumptionsiFF1} relates to the
 prior distribution of the weights, and
probably could be relaxed
to allow for more general smooth distributions; here again, using non-Gaussian weights would increase the complexity of the computations, because
the Hessian of the (log of the) prior will not be deterministic.
\iffalse
This assumption can be relax to any smooth distribution yielding more elaborate Fisher matrix for the prior. We note that this limit the space of function that the data represents. Exploring a larger function class can be done also by considering a multi layer model with Gaussian prior.
\fi

Our last assumption is a technical assumption that allows for the use of calculus in the derivation of our bounds.
It can be relaxed as in e.g., \cite{mei2019generalization}. We note
that Assumption \ref{subsec:assumptionsiFF2}
is satisfied by all commonly-used activation functions,
including the ReLU and sigmoid functions.
\begin{assumption}
\label{subsec:assumptionsiFF2}
 The activation function $\sigma$ is five times  differentiable and bounded, together with its derivatives,
	by some constant $K_\sigma<\infty$.
\end{assumption}

For convenience, we introduce the constants
\begin{gather}
\eta_{0}(v)=\mathbb{E}\left[\sigma^{2}(vz)\right]-\mathbb{E}\left[\sigma(vz)\right]^{2},\quad %	\eta_{1}(v)=
\xi(v)=	\left(v\mathbb{E}\left[\sigma'(vz)\right]\right)^{2},
\label{eq: constants} \\
\eta_{1}(v)=v^{2}\mathbb{E}\left[\sigma'(vz)^{2}\right]
\end{gather}
where we recall that $z\sim\mathcal{N}(0,1)$, and $v={\sigma_x}/{\sqrt{\alpha}}$.
During our analysis we will sometime use for brevity 
$\xi(v)=\xi$, and $\eta_{0}(v)=\eta_0$.

\section{Main results \label{sec: main results}}
As discussed in the introduction,
we  provide lower bounds on the generalization error of general estimator for feed-forward networks. For completeness, we also
show a bound for unbiased estimators for the general
learning model of Section \ref{sec-genlearning}.
Our bounds are based on
both Bayesian and non-Bayesian versions of the CR bound;
we review the background in Appendix \ref{Sec:CRB}.
The following is our main result.
\iffalse
To derive the b, we use a Bayesian version of the CR bound. The unbiased bound is derived using the classical CR. A review of the CR bounds for unbiased estimator and general estimator is provided in Section \ref{Sec:CRB}. %\subsection{Fully connected neural network \label{subsubsec: nonlinear_bounds_any_estimator}}
%We next turn to the case of true feed-forward networks, that is with nonlinear activation functions.
\fi
\begin{theorem}
		Consider the model
in \eqref{eq:model}, (\ref{eq: deep model}) with $L=2$, and let Assumptions \ref{subsec:assumptions},  \ref{subsec:assumptionsiFF1} and \ref{subsec:assumptionsiFF2}
hold. Let $\hat{y}(\tilde{\boldsymbol{x}},Z)$ be an arbitrary  measurable and square integrable  estimator. Then,
 in the regime $M,N_1,d\rightarrow \infty$ such that $\beta_{1}=\lim_{d\to\infty} {N_1}/d\in (0,\infty)$
	and $\gamma_0=\lim_{d\to\infty} {d}/M\in(0,\infty)$, it holds that
\begin{equation}
\mathbb{E}\left(\tilde{y}-\hat{y}\right)^{2}\geq  \max\left(B^{(1)},B^{(2)}\right)\label{eq:Layer_wise_L=2},
\end{equation}
where
\begin{equation}
	B^{(1)}=
\sigma_{\varepsilon}^{2}\xi\int\frac{1+\xi(1-\gamma_{0}^{-1})s/(\sigma_{\varepsilon}^{2}\alpha_{2})}{\xi s+\alpha_{2}\sigma_{\varepsilon}^{2}}d\rho_{\gamma_{0}^{-1}}(s)+o(1),
	\label{eq:B1Tight}
\end{equation}
and with $a_1=m_1(\sigma_{\varepsilon}\sqrt{{\alpha_2}{\beta_{1}}})$, and $a_2=m_2(\sigma_{\varepsilon}\sqrt{{\alpha_2}{\beta_{1}}})$,
	\begin{equation}
B^{(2)}=\frac{\sigma_{\varepsilon}}{\sqrt{\alpha_2\beta_{1}}}a_{1}\left(\frac{\xi}{1+\xi a_{1}a_{2}}+\eta_{0}-\xi\right)+o(1)\label{eq:B2Tight}.
\end{equation}	
Here, the functions $m_1=m_{1}(u),m_2=m_{2}(u):\mathbb{C}_{+}\rightarrow\mathbb{C}_{+}$,
satisfy for $\Re(u)>0$:	
\[
m_{1}={\beta_{1}}\left(u+\left(\eta_{0}-\xi\right)m_{2}+\frac{\xi m_{2}}{\xi m_{1}m_{2}+1}\right)^{-1},
\]	
\[
m_{2}=\frac{1}{\gamma_{0}}\left(u+\left(\eta_{0}-\xi\right)m_{1}+\frac{\xi m_{1}}{\xi m_{1}m_{2}+1}\right)^{-1}
.\]\label{Thm:B1B2VanTrees}
\label{Thm: B1_B2}
\end{theorem}
\noindent
We note that the fixed point equations  for $m_1,m_2$
have a unique solution, see 
in \cite[Lemma 10.2]{mei2019generalization} in the domain relevant for the lower bound.

The proof of Theorem \ref{Thm: B1_B2} is presented in
Section \ref{proof of the any estimator}. We provide here
several amplifying remarks.
\begin{remark} \label{rem:scaling} It is instructive to consider
  several limiting cases in Theorem \ref{Thm: B1_B2}. In the high SNR limit,
  that is  $\sigma_\varepsilon\to 0$,
  the parameter $\gamma_0$ plays a crucial role: in the
  underparametrized regime  $\gamma_0<1$, as expected we get that $B^{(1)},B^{(2)}\to_{\sigma_{\varepsilon}\to 0} 0$. On the other hand, in the overparametrized regime $\gamma_0>1$, we find that
  %s expected, when $\sigma(\varepsilon)\to 0$
  %then the
  %Our bounds are calculated under the assumption of existence of additive noise in the data. Interestingly, taking the variance of the noise to zero in our bounds (Eq. (\ref{eq:B1Tight}) and Eq. (\ref{eq:B2Tight})), we find that the bound
  $B^{(1)}\to_{\sigma_{\varepsilon} \rightarrow 0}  \xi (1-\gamma_0^{-1})/\alpha_2>0$, while still $B^{(2)}\to_{\sigma_{\varepsilon} \rightarrow 0} 0$.
  In the low SNR
  limit $\sigma_{\varepsilon}\to \infty$, both $B^{(1)}$ and $B^{(2)}$ converge to non-zero limits (equal to $\xi/\alpha_2$ and $\eta_0/\alpha_2$, respectively). By remark \ref{rem:eta_xi}, it follows that in that limit, $B^{(2)}$ is a better bound.
\end{remark}
\begin{remark} \label{rem:eta_xi}
	The Cauchy-Schwarz inequality together with a Gaussian integration by parts show that as soon as $\sigma$ is nonlinear, one has that
$\xi=v^2(\mathbb{E}[\sigma'(vz)])^2=(\mathbb{E}[z\sigma(vz)])^2\leq \mathbb{E}[\sigma(vz)^2]=\eta_0$, and since  $m_1(u), m_2(u)$ are positive,
 we have that $B^{(2)}> 0$.
\end{remark}
\begin{remark}
	The bound $B^{(1)}$ does not depend on the size of the hidden layer, as long as it is of the order of the first hidden layer, see Figure \ref{fig:Bounds_N}. This shows that in this regime one cannot recover the matrix $W^{(1)}$
	with asymptotically vanishing error.
\end{remark}

\begin{remark}
	By construction, the bound $B^{(2)}$ is also a bound on a
	general estimator of the random features model in which only the
	weights of the last layer are being learned.
\end{remark}
\begin{remark}
This theorem can be generalized to data generated by a multi-layer neural network, which represents a larger class of functions under the Gaussian assumption on the weights. This can be done by using Corollary \ref{Cor: Layer-wise bound} which presents the general expression for the bound in the multi-layer case.
\end{remark}
\begin{remark}
To put our result in perspective, we note
that there are several generalization upper bounds
for neural networks with ReLU activation function.
The VC dimension for such neural networks provides
essentially parameter counts, i.e. $O(LN/\sqrt{M})$ where $L$ is the number of layers, $N$ is the number of hidden units in each layers, and $M$ is the number of samples \cite{bartlett2019nearly}.
\iffalse
This provides generalization guarantees
for $M\gg N^{2}$.
\fi
These results are
improved in  \cite{bartlett2017spectrally}, where the factor $N$
is replaced by several norms of the layer weights. There are
further results,
derived
using the PAC-Bayesian framework
%generalization bounds
for ReLU activation, see
\cite{neyshabur2017pac,dziugaite2017computing}. These
analyze generalization behavior in neural networks,
and analytically derive a margin-based bound in terms of norms of the weights,
or calculate the entire bound empirically given the data.
In our setup, to be useful all these estimates require that $M\gg N^2$.
In contrast, we provide analytical expression for
lower bounds on the generalization error. These bounds are on the MSE loss
for general activation function, and in the regime $M\sim N$.
We note that we allow for general estimators,
while introducing special assumption on the structure of the data,
such as a Gaussianity assumption and
the existence of additive noise in the training samples.
%This can be equivalent to the margin in the above bounds.
%Our bounds show in particular
%that there is a constant error of order one when
%$M,N_{1},d$ are large and linearly proportional to one another.
\end{remark}

\subsection{Special cases\label{subSec: Special cases}}
\subsubsection{Linear activation \label{subSec: Linear regression}}
For linear activation,
%In this case, under the linear width assumption, 
it is a straight forward to show that the optimal estimator in the MMSE sense
%in the regime $M,N_1,d\rightarrow \infty$ such that $\beta_{1}=\lim_{d\to\infty} {N_1}/d\in (0,\infty)$
%and $\gamma_0=\lim_{d\to\infty} {d}/M\in(0,\infty)$
takes the following form:
\begin{align}
	\hat{y}_{\mathrm{opt}}
%	\mathbb{E}\left[\tilde{y}|Z,\tilde{\boldsymbol{x}}\right]
%	=\frac{1}{N_{1}d\alpha}\mathbb{E}\left[\left\Vert \boldsymbol{w}^{(2)}\right\Vert ^{2}\left(X\left(\frac{1}{N_{1}d\alpha}\left\Vert \boldsymbol{w}^{(2)}\right\Vert ^{2}X^{T}X+I_{M}\sigma_{\varepsilon}^{2}\right)^{-1}Y\right)^{T}|Z\right]\tilde{\boldsymbol{x}}
	=\frac{1}{\alpha\alpha_{2}d}Y^{T}\left(\frac{1}{\alpha\alpha_{2}d}X^{T}X+I_{M}\sigma_{\varepsilon}^{2}\right)^{-1}X^{T}\tilde{\boldsymbol{x}}
\end{align}
The minimum mean square error is then:
\begin{align}
\mathrm{MMSE} =\mathbb{E}(\hat{y}_{\mathrm{opt}}-\tilde{y})^{2}
=\sigma_{\varepsilon}^{2}v^{2}\left(\int\frac{v^{2}s(1-1/\gamma_{0})/\left(\sigma_{\varepsilon}^{2}\alpha_{2}\right)+1}{v^{2}s+\sigma_{\varepsilon}^{2}\alpha_{2}}d\rho_{\gamma_{0}^{-1}}\right) = B^{(1)}.
\end{align}
Therefore, the bound is tight in this case.
\subsubsection{Highly over-parameterized regime}
We explore two extreme cases:
\paragraph{Finite samples size}
In this regime $\beta_{1}=\lim_{d\to\infty} {N_1}/d\in (0,\infty)$ and $M$ finite i.e. $\gamma_0\rightarrow \infty$. Applying Theorem \ref{Thm: B1_B2} with $\gamma_0\rightarrow \infty$	we obtain $B^{(1)}={\xi}/{\alpha_2}$, and $B^{(2)}={\eta_{0}}/{\alpha_2}$. By Remark \ref{rem:eta_xi}, $\eta_{0}>\xi$, therefore, the bound takes the following form:
%\begin{corollary}
%	Consider the model
%	in \eqref{eq:model}, (\ref{eq: deep model}) with $L=2$, and let Assumptions \ref{subsec:assumptions},  \ref{subsec:assumptionsiFF1} and \ref{subsec:assumptionsiFF2}
%	hold. Let $\hat{y}(\tilde{\boldsymbol{x}},Z)$ be an arbitrary  measurable and square integrable  estimator. Then,
%	in the regime $N_1,d\rightarrow \infty$ such that $\beta_{1}=\lim_{d\to\infty} {N_1}/d\in (0,\infty)$ and is $M$ is finite, it holds that
\begin{equation}
\mathbb{E}\left(\tilde{y}-\hat{y}\right)^{2}\geq  \frac{\eta_{0}}{\alpha_2}
\end{equation}
%\end{corollary}
Interestingly, this does not depend on the variance
of the noise. This is also a case for which $B^{(2)}>B^{(1)}$ for non linear activation function.
\paragraph{Infinite samples size and input size}
In this regime, the number of neurons and input size, $N_1,d$, diverges together but $N_1$ is larger than any constant times $d$, such that $\gamma_0=\lim_{d\to\infty} {d}/M\in(0,\infty)$, and $\beta_1 \rightarrow \infty$. Applying Theorem \ref{Thm: B1_B2}, we obtain that  $B^{(2)}=\eta_{0}/{{\alpha_2}}+o(1)$. Since, the bound $B^{(1)}$ does not depend on the size of the number of hidden units $N_1$ as long as it is large, we obtain the following results:
%\begin{corollary}
%	Consider the model
%	in \eqref{eq:model}, (\ref{eq: deep model}) with $L=2$, and let Assumptions \ref{subsec:assumptions},  \ref{subsec:assumptionsiFF1} and \ref{subsec:assumptionsiFF2}
%	hold. Let $\hat{y}(\tilde{\boldsymbol{x}},Z)$ be an arbitrary  measurable and square integrable  estimator. Then,
%	in the regime $M,N_1,d\rightarrow \infty$ such that $\gamma_0=\lim_{d\to\infty} {d}/M\in(0,\infty)$, it holds that
\begin{equation}
\mathbb{E}\left(\tilde{y}-\hat{y}\right)^{2}\geq  	\max\left(\frac{\eta_{0}}{{\alpha_2}},\sigma_{\varepsilon}^{2}\xi\int\frac{1+\xi(1-\gamma_{0}^{-1})s/(\sigma_{\varepsilon}^{2}\alpha_{2})}{\xi s+\alpha_{2}\sigma_{\varepsilon}^{2}}d\rho_{\gamma_{0}^{-1}}(s)\right)+o(1),
\end{equation}
%\end{corollary}

%\begin{corollary}
%	Consider the model
%	in \eqref{eq:model}, (\ref{eq: deep model}) with $L=2$ with $\sigma(x)=x$, and let Assumptions \ref{subsec:assumptions},  \ref{subsec:assumptionsiFF1} and \ref{subsec:assumptionsiFF2}
%	hold. Let $\hat{y}(\tilde{\boldsymbol{x}},Z)$ be an arbitrary  measurable and square integrable  estimator. Then,
%	in the regime $M,N_1,d\rightarrow \infty$ such that $\beta_{1}=\lim_{d\to\infty} {N_1}/d\in (0,\infty)$
%	and $\gamma_0=\lim_{d\to\infty} {d}/M\in(0,\infty)$, it holds that
%	\begin{equation}
%	\mathbb{E}\left(\tilde{y}-\hat{y}\right)^{2}\geq  \max\left(B^{(1)},B^{(2)}\right)\label{eq:Layer_wise_L=2},
%	\end{equation}
%	where
%	\begin{equation}
%	B^{(1)}=
%	\sigma_{\varepsilon}^{2}v^2\int\frac{1+v^2(1-\gamma_{0}^{-1})s/\sigma_{\varepsilon}^{2}\alpha_{2}}{v^2s+\alpha_{2}\sigma_{\varepsilon}^{2}}d\rho_{\gamma_{0}^{-1}}(s)+o(1),
%	\label{eq:B1Tight}
%	\end{equation}
%	and with $a_1=m_1(\sigma_{\varepsilon}\sqrt{{\alpha_2}{\beta_{1}}})$, and $a_2=m_2(\sigma_{\varepsilon}\sqrt{{\alpha_2}{\beta_{1}}})$,
%	\begin{equation}
%	B^{(2)}=\frac{\sigma_{\varepsilon}}{\sqrt{\alpha_2\beta_{1}}}\frac{a_{1}v^2}{1-v^2a_{1}a_{2}}+o(1)\label{eq:B2Tight}.
%	\end{equation}	
%	Here, the functions $m_1=m_{1}(u),m_2=m_{2}(u):\mathbb{C}_{+}\rightarrow\mathbb{C}_{+}$,
%	satisfy for $\Re(u)>0$:	
%	\[
%	m_{1}=-{\beta_{1}}\left(u+\frac{v^2 m_{2}}{v^2m_{1}m_{2}-1}\right)^{-1},
%	\]	
%	\[
%	m_{2}=\frac{1}{\gamma_{0}}\left(u-\frac{v^2m_{1}}{v^2m_{1}m_{2}-1}\right)^{-1}
%	.\]\label{Thm:B1B2VanTrees}
%	\label{Thm: B1_B2}
%\end{corollary}

Note that in general whether or not  $B^{(1)}>B^{(2)}$ in Theorem \ref{Thm: B1_B2} depends on the activation function and the model parameters. Interestingly, in the regime of high SNR and overparametrization (see Remark \ref{rem:scaling}),
$B^{(1)}$ tends to be better, this is also observed in Figure. \ref{fig:Comparison-of-the bounds_SNR}.
In Figures \ref{fig:Comparison-of-the bounds_SNR},  \ref{fig:Comparison-of-the bounds_gamma}, and \ref{fig:Bounds_N} below we compare the bounds for different Signal to Noise Ratio (SNR),  with $SNR:=10 \log_{10} (\sigma_x^2/(\alpha \alpha_2 \sigma_{\varepsilon}^2))$, different $\gamma_0$,
and different $\beta_1$. The definition of the SNR is with respect to the total variance of a linear model,  $f_{\boldsymbol{\theta}}(\boldsymbol{x})=\frac{1}{\sqrt{dN_{1}}}(\boldsymbol{w}^{(2)})^{T}W^{(1)}\boldsymbol{x}$, divided by the noise.
 The comparison is presented for four activation functions,
 Linear: $\sigma(x)=x$, Sigmoid: $\sigma(x)=1/(1+e^{-x})$, Tanh: $\sigma(x)=\mathrm{tanh}(x)$, and Relu: $\sigma(x)=\max(x,0)$.
That there is no clear ``winner'' is consistent with the analysis in \cite{bobrovsky1987some}, in which the authors consider similar types of bounds 
not in the context of neural networks. They compare the bounds in cases where some of the parameters are being conditioned and show via toy examples that predicting which bound is better depends on the specifics of the problem.
\begin{figure}[h!]
	\begin{centering}
		\includegraphics[trim={1cm 0cm 1cm 0cm},scale=0.45]{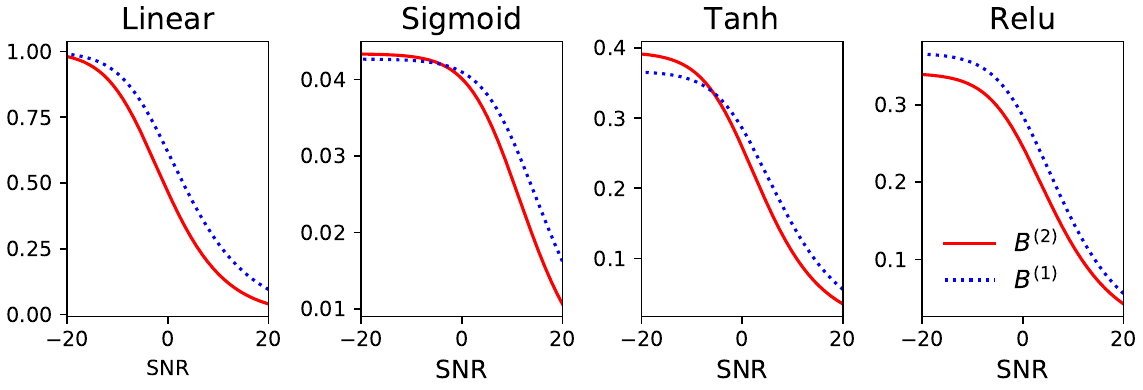}
		\par\end{centering}
	\caption{\label{fig:Comparison-of-the bounds_SNR}Comparison of the
		bounds in Theorem \ref{Thm:B1B2VanTrees} for 4 activation functions: Relu, Sigmoid, Tanh, and Linear, as function of SNR.
	Here,  $N_1=M=50$, $d=50$, $\alpha=1$,$\alpha_{2}=2$, and
		$\sigma_{x}^{2}=1$. The dashed blue line is (\ref{eq:B1Tight}). The solid  red line is (\ref{eq:B2Tight}).  }	
\end{figure}

\begin{figure}[h!]
	\begin{centering}
		\includegraphics[trim={1cm 0cm 1cm 0cm},scale=0.45]{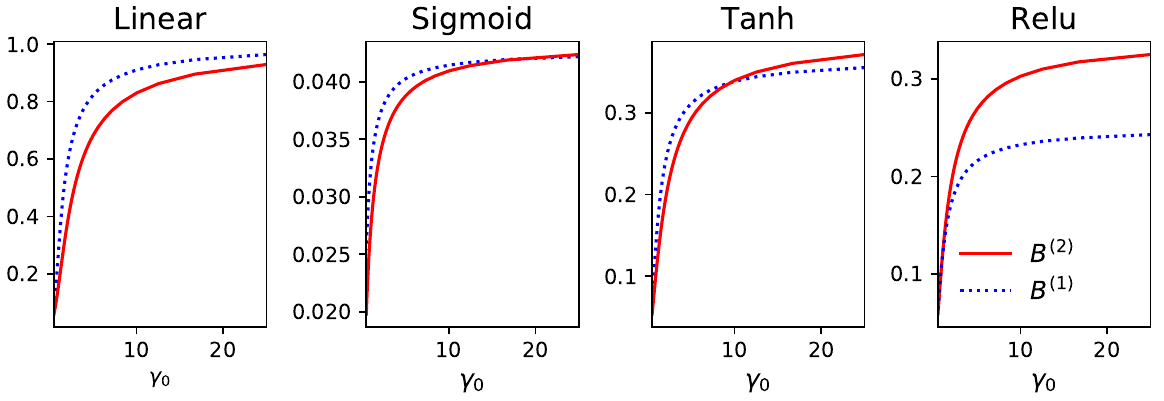}
		\par\end{centering}
	\caption{\label{fig:Comparison-of-the bounds_gamma}Comparison of the	bounds in \ref{Thm:B1B2VanTrees}, as function of $\gamma_0$.
Here, $N_1=100$, $d=100$, $\alpha=\alpha_{2}=1$, $\sigma_{\varepsilon}^{2}=0.1$ and
		$\sigma_{x}^{2}=1$. The dashed blue line is (\ref{eq:B1Tight}). The solid  red line is (\ref{eq:B2Tight}).  }	
\end{figure}

\begin{figure}[h!]
	\begin{centering}
		\includegraphics[trim={1cm 0cm 1cm 0cm},scale=0.45]{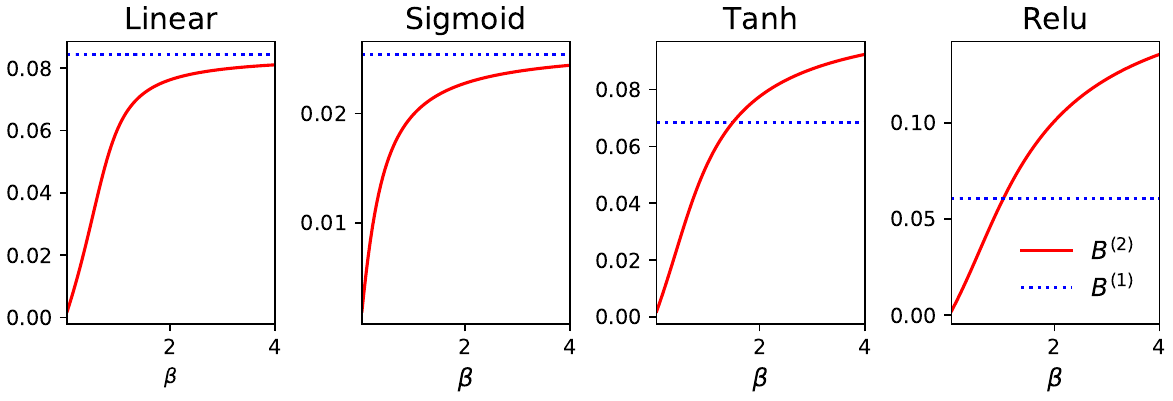}
		\par\end{centering}
	\caption{\label{fig:Bounds_N} Comparison of the bounds in Theorem \ref{Thm:B1B2VanTrees} as a function of the  $\beta = \frac{N_1}{d}$.
	Here $N_1=d=50$, $\alpha=\alpha_{2}=1$, and
		$\sigma_{x}^{2}=1$. The dashed blue line is(\ref{eq:B1Tight}). The solid red line is (\ref{eq:B2Tight}).}	
\end{figure}
\subsection{Unbiased estimators}
So far we discussed
%In the analysis above we use
the Bayesian version of the CR bound. In this subsection we provide as a
comparison the classical version of the CR bound,
which does not require knowledge about the prior distribution of the
parameters,  but makes however an unbiasedness assumption on the estimator.
Recall that
\iffalse
It is convenient to decompose the generalization error as
\begin{equation}
\mathbb{E}_{\tilde{\boldsymbol{x}},Z,\boldsymbol{\theta}}\left(\tilde{y}-\hat{y}\right)^{2}=\underset{\mathrm{Variance}}{\underbrace{\mathbb{E}_{\tilde{\boldsymbol{x}},\boldsymbol{\theta}}\mathrm{Var}_{Z|\boldsymbol{\theta},\tilde{\boldsymbol{x}}} \left[\hat{y}\right]}}+\underset{\mathrm{Bias}}{\underbrace{\mathbb{E}_{\tilde{\boldsymbol{x}},\boldsymbol{\theta}}\left[b_{\boldsymbol{\theta}}(\tilde{\boldsymbol{x}})^{2}\right]}},
\label{eq: bias variance decompossion}
\end{equation}
where 
\fi 
the bias of an estimator 
$\hat y(\tilde{\boldsymbol{x}},Z)$
is $b_{\boldsymbol{\theta}}(\tilde{\boldsymbol{x}})=f_{\boldsymbol{\theta}}(\tilde{\boldsymbol{x}})-\mathbb{E}_{Z|\boldsymbol{\theta},\tilde{\boldsymbol{x}}}\hat{y}(\tilde{\boldsymbol{x}},Z)$. \textit{Unbiased estimators} are those which satisfy $b_{\boldsymbol{\theta}}(\tilde{\boldsymbol{x}})=0$, almost surely.
Note that the minimum mean square error (MMSE) estimator
$\hat{y}_{\mathrm{opt}}=\mathbb{E}\left[\tilde{y}|\tilde{\boldsymbol{x}},Z\right]$ is typically not unbiased.

%But in order to evaluate the bound one needs to assume the estimator is unbiased unless there exists some prior knowledge on the bias of the estimator which is rarely accessible.
As experience shows, in high dimension biased estimators can significantly
outperform biased ones. Theorem \ref{Thm:CRLB_unbiased},
whose proof  is provided in Section \ref{Sec:CRLB_unbiased_proof},
together with Theorem \ref{Thm:deep_nonlinear_unbiased} in the appendix,
show that indeed unbiased estimators have bad performance in the truly nonlinear setup; see Remark \ref{rem:B.4}.

The CR bound for unbiased estimators involves
%In Appendix \ref{Sec:CRB} we provide the two versions of the CR bound. The assumption of unbiasedness is not a reasonable requirement in high dimension the theorem below emphasis it.
%In order to calculate the lower bound for the generalization error of unbiased estimators, we first present
the Fisher information $I(\boldsymbol{\bm{\theta}})\in \mathbb{R}^{P\times P}$, which in our case
%if we plug in the distribution $p(Z|\bm{\theta})$,
takes the  form
\begin{equation}
I(\boldsymbol{\theta}) = \frac{M}{\sigma_{\varepsilon}^2}A(\boldsymbol{\theta}),\quad
A(\boldsymbol{\theta})=\mathbb{E}_{\boldsymbol{x}}\left[\nabla_{\boldsymbol{\theta}} f_{\boldsymbol{\theta}}(\boldsymbol{x})\nabla_{\boldsymbol{\theta}} f_{\boldsymbol{\theta}}(\boldsymbol{x})^T\right].	\label{eq:Fisher_model}
\end{equation}
\begin{theorem}
	\label{Thm:CRLB_unbiased} Consider the model in \eqref{eq:model}.
	Assume that $\nabla_{\boldsymbol{\theta}} f_{\boldsymbol{\theta}}(\boldsymbol{x})$ exists and is square integrable under
	$p(\boldsymbol{x})$. Let $\hat{y}(\tilde{\boldsymbol{x}},Z)$ be any square integrable unbiased estimator. Then,
	\begin{equation}
	\mathbb{E}_{Z,\tilde{\bm{x}}|\bm{\theta}}\left[\left(\tilde{y}-\hat{y}\right)^{2}\right]\geq \sigma_{\varepsilon}^{2}\frac{\mathrm{Rank}\left(I(\boldsymbol{\theta})\right)}{M}.
	\label{eq:CRLB_unbiased}
	\end{equation}
\end{theorem}
The proof of Theorem \ref{Thm:CRLB_unbiased} is provided in Section \ref{Sec:CRLB_unbiased_proof}. In Appendix \ref{Sec:rankI} we evaluate the rank of the Fisher information in a few examples. We show that the reason for high rank (and then bad performance) is a combination of high dimension phenomena with
the non-linearity of the model.
\section{Performance evaluation\label{sec:numerics}}
In this section, we compare the bounds of  Theorem  \ref{Thm: B1_B2} to the performance of estimators used in practice.
%\subsection{Linear regression - the ridge regression estimator}\label{sec:LSE}
%We consideer the setup of Theorem \ref{Thm:absulutLR}. Consider the following estimator
%$\hat{y}(\tilde{\boldsymbol{x}},Z)_\mathrm{LS}=\hat{\boldsymbol{\theta}}^{T}\tilde{\boldsymbol{x}}/\sqrt{d}$,
%where $\hat{\boldsymbol{\theta}}=(XX^{T}/d+ \lambda I_{d})^{-1}X^{T}\boldsymbol{y}^T/\sqrt{d}$,
%and $\lambda$ is a regularization parameter. The generalization error for $ \lambda>0$ in this setting reads (see for example
%\cite[Theorem 5]{hastie2019surprises} or \cite[Proposition 1]{dicker2016ridge}),
%\begin{equation}
%\mathbb{E}\left(\tilde{y}-\hat{y}_\mathrm{LS}\right)^{2}={\sigma^2_x}{\gamma_{0}}\int\frac{ \lambda^{2}/(\gamma_{0}\alpha)+\sigma_{\varepsilon}^{2}s}
%{\left(s+ \lambda\right)^{2}}d\rho_{\gamma_0}(s)+o(1)
%\end{equation}
%By
%\cite[Theorem 5]{hastie2019surprises},
%the minimum error is achieved at $\lambda^{\mathrm{opt}}=\sigma_{\varepsilon}^{2}\alpha\gamma_{0}$.
%Comparing this test error at $ \lambda^{\mathrm{opt}}$ to the bound in Theorem \ref{Thm:absulutLR} shows that this estimator asymptotically  achieves the bound.
%\subsection{Stochastic gradient descent}
Since an analytic expression for the optimal estimator is not known, 
%is impossible in the case of nonlinear activation function to our knowledge, 
we instead compare numerically the bounds to the performance of an estimator based on the stochastic gradient descent algorithm\cite{robbins1951stochastic}, which is one of the most common algorithm used to train neural network models.
We do so in the framework of the
``teacher-student model'' \cite{seung1992statistical,engel2001statistical}, which we describe next.
\subsection{Data generation (teacher network)}
Following the ``teacher-student model'', we generate numerically $N_{\theta}$ realization of the teacher true weights $\boldsymbol{\theta}=(\boldsymbol{w}^{(2)};W^{(1)})$ of a two layers neural network. The weights vector $\bm{\theta}$ is drawn from a Gaussian distribution with zero mean and variance $\alpha=1$ for each element in the matrix $W^{(1)}$, and variance $\alpha_2=1$ for each element of the vector $\bm{w}^{(2)}$. For each true realization of $\bm{\theta}$, we generate $N_d$ data-sets of size $M$, and a test set of size
$N_s$. Each training sample $\boldsymbol{x}^{(i)}\sim\mathcal{N}(0,\sigma^2_xI_d)$, and $y^{(i)}=(\boldsymbol{w}^{(2)})^T\sigma(W^{(1)}\boldsymbol{x}^{(i)}/\sqrt{d})/\sqrt{N_1}+\varepsilon^{(i)}$, where $\varepsilon^{(i)}\sim \mathcal{N}(0,\sigma^2_{\varepsilon})$.
\subsection{Estimator (student network)}
Given the data generated by the teacher network, we use a two layers neural network with twice as large number of neurons in the hidden layer then the teacher network. The estimator is obtained by
running a stochastic gradient descent (SGD) algorithm \cite{robbins1951stochastic} for $N_{\mathrm{epochs}}=100$ epochs and batch size $N_b=5$ with learning rate $\eta=0.5/M$ for Sigmoid and $\eta=0.01/M$ for Tanh, linear and Relu. The generalization error is obtained by averaging over all these quantities.

Figure \ref{fig:SGD_vs_bound_sn_small} presents the results for two variances
$\sigma^2_\varepsilon$ of the noise. In Figure \ref{fig:SGD_vs_bound_Uniform_weights}, the bounds are compared to the performance of SGD with data generated by a teacher network with the same size as before but with uniform prior on the weights in both layers. The variance of the weights matches the variance in the Gaussian case. In this case, even though the weights are not drawn form a Gaussian prior, the numerical results suggest that the bounds do apply. In both figures we also compare the performance of SGD to the performance of a trivial mean predictor, $\hat{y}_\mathrm{mean}=\sum^M_{i=1}y_i/M$.

\begin{figure}[h!]
	\begin{centering}
		\includegraphics[trim={1cm 0cm 1cm 0cm},scale=0.4]{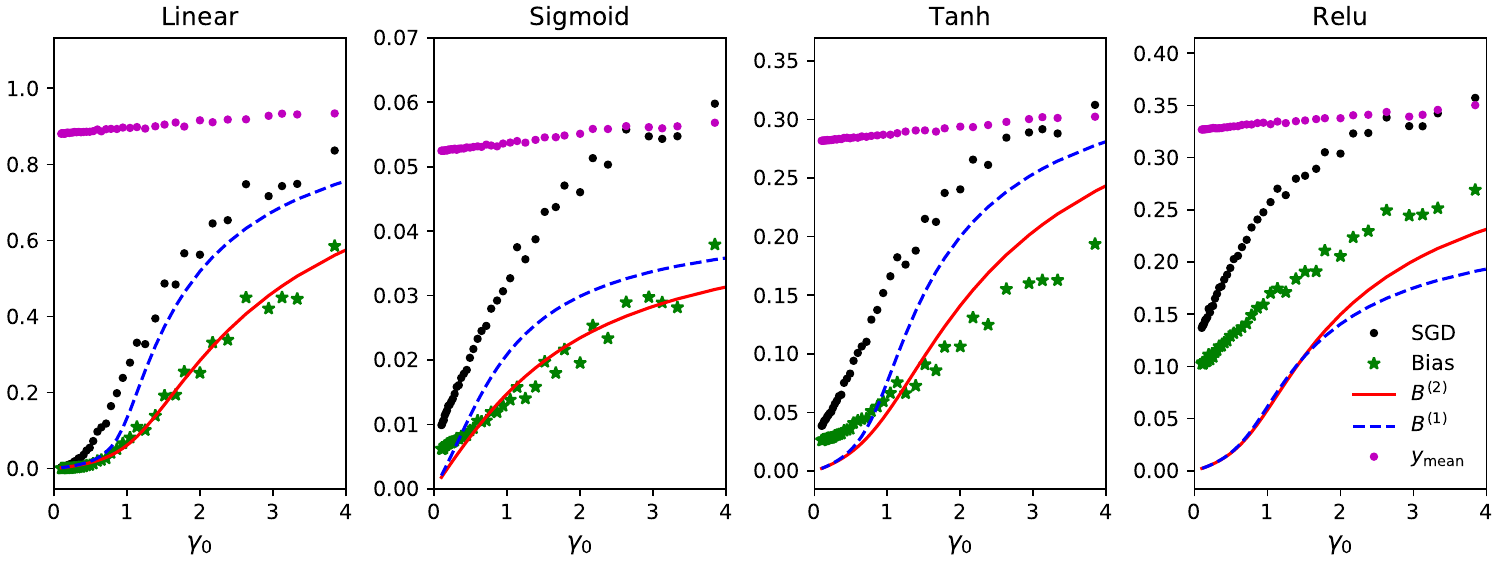}
		\par\end{centering}
\begin{centering}
		\includegraphics[trim={1cm 0cm 1cm 0cm},scale=0.4]{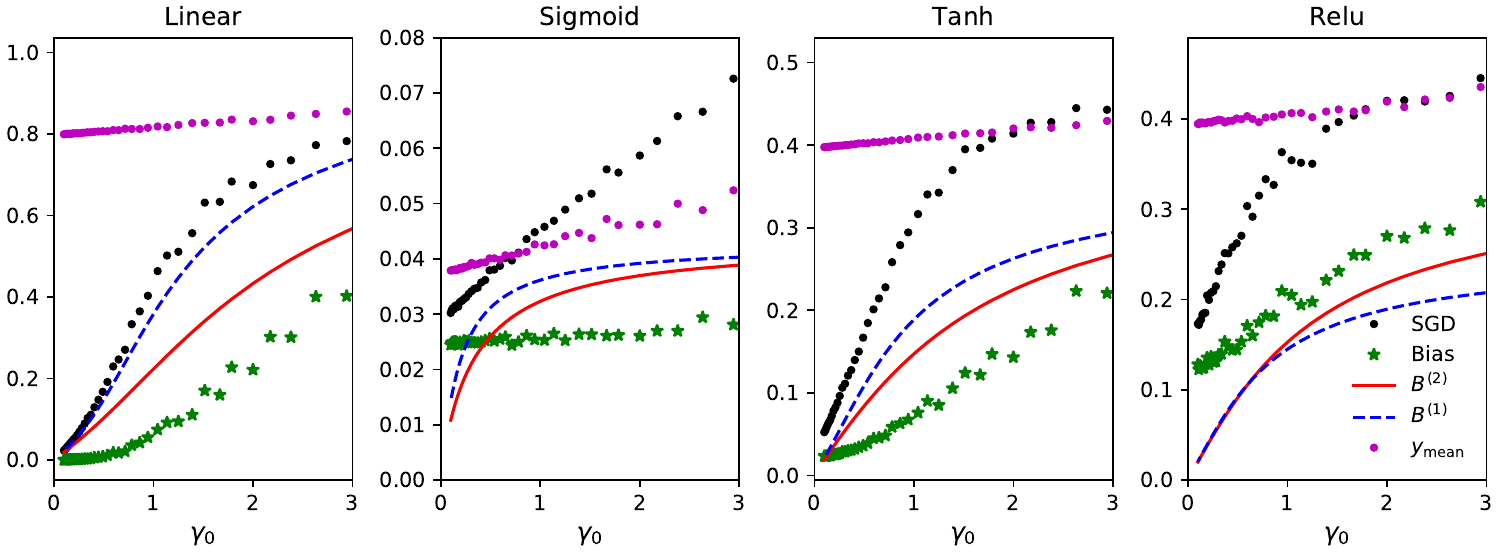}
		\par\end{centering}
	\caption{\label{fig:SGD_vs_bound_sn_small}The generalization performance of
		SGD and the bounds in Theorem \ref{Thm: B1_B2} as  function of $\gamma_{0}={d}/{M}$.
		The generalization error of SGD is in black dots, and its bias in green dots. The performance of trivial mean predictor are in magenta. The dashed blue line is  the bound (\ref{eq:B1Tight}), while the
red solid line is (\ref{eq:B2Tight}). In the top panel, $\sigma_{\varepsilon}^{2}=0.02$,and in the bottom panel,  $\sigma_{\varepsilon}^{2}=0.2$. Other parameters are
		$N_1=d=50$,$\alpha=\alpha_{2}=\sigma_{x}^{2}=1$.}
\end{figure}

\begin{figure}[h!]
	\begin{centering}
		\includegraphics[trim={1cm 0cm 1cm 0cm},scale=0.4]{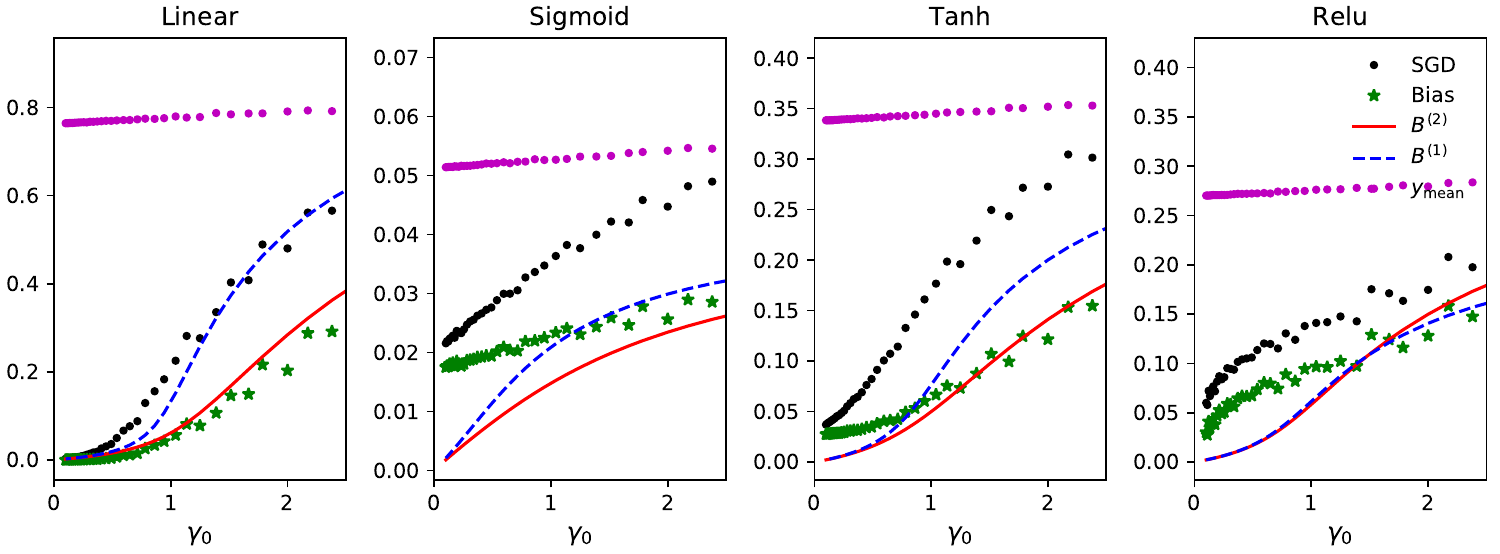}
		\par\end{centering}
	\caption{\label{fig:SGD_vs_bound_Uniform_weights}The generalization performance of
		SGD vs the bounds in Theorem \ref{Thm: B1_B2} as function of $\gamma_{0}={d}/{M}$. The data given by the teacher is generated with a uniform prior on the weights, i.e. $W_{ij}^{(1)},{w}_i^{(2)}\sim U(-b,b)$, where $b=\sqrt{3}/\alpha$.
		The generalization error of SGD is in black dots, and its bias in green dots. The performance of trivial mean predictor are in magenta. The dashed blue line is the bound (\ref{eq:B1Tight}), while the
		red solid line is (\ref{eq:B2Tight}).  The parameters are
		$N_1=d=50$, $\alpha=\alpha_{2}=\sigma_{x}^{2}=1$ and $\sigma_{\varepsilon}^{2}=0.02$. The student network is with $150$ hidden units.}
\end{figure}

\section{Proof of the main Theorems\label{sec: proofs}}
We recall elements of the
Cram\'{e}r-Rao theory in Appendix \ref{Sec:CRB}.

\subsection{Proof of Theorem \ref{Thm: B1_B2}}\label{proof of the any estimator}
We will use Corollary \ref{Cor: Layer-wise bound} for depth $L=2$, which states that
\begin{equation}
	E_{g}\geq \max(B^{(1)},B^{(2)})
\end{equation}
Recall that following  (\ref{eq: deep model}) for $L=2$ and  $N_1=N$,  $f_{\boldsymbol{\theta}}(\boldsymbol{x})=(\boldsymbol{w}^{(2)})^T\sigma(\boldsymbol{q})/\sqrt{N}$,
and $\boldsymbol{q}=W^{(1)}\boldsymbol{x}/\sqrt{d}$. The
gradient is then
\[
\nabla_{\boldsymbol{\theta}}f_{\boldsymbol{\theta}}(\boldsymbol{x})=\left[\begin{array}{c}
\nabla_{W^{(1)}}f_{\boldsymbol{\theta}}(\boldsymbol{x})\\
\nabla_{\boldsymbol{w}^{(2)}}f_{\boldsymbol{\theta}}(\boldsymbol{x})
\end{array}\right]=\left[\begin{array}{c}
\frac{1}{\sqrt{Nd}}D^{1}\boldsymbol{w}^{(2)}\otimes\boldsymbol{x}\\
\frac{1}{\sqrt{N}}\sigma(\boldsymbol{q})
\end{array}\right],
\]
where $D^{1}=\mathrm{diag}(\sigma'(\boldsymbol{q}))$.
In the following, we  evaluate separately $B^{(1)}$ and $B^{(2)}$.

\paragraph{Evaluation of $B^{(1)}$:}
Take $\boldsymbol{\theta}_{l}=\boldsymbol{w}^{(2)}$ in Corollary  \ref{Cor: Layer-wise bound}. Then, the bound in
 (\ref{eq:B^{(l)}}) reads
\begin{equation}
	B^{(1)}=\frac{\sigma_{\varepsilon}^{2}}{d}\mathbb{E}_{\boldsymbol{w}^{(2)},X}\left[\mathrm{Tr}\left(\left(I(X,\boldsymbol{w}^{(2)})+\sigma_{\varepsilon}^{2}\alpha I_{Nd}\right)^{-1}Q(\boldsymbol{w}^{(2)})\right)\right],\label{eq:B1_ba}
\end{equation}
where
\begin{multline}
Q(\boldsymbol{w}^{(2)})=d\mathbb{E}_{\tilde{\boldsymbol{x}}}\left[\mathbb{E}_{W^{(1)}}\nabla_{W^{(1)}}f_{\boldsymbol{\theta}}(\tilde{\boldsymbol{x}})\left(\mathbb{E}_{W^{(1)}}\nabla_{W^{(1)}}f_{\boldsymbol{\theta}}(\tilde{\boldsymbol{x}})\right)^{T}\right]\\
=\frac{1}{N}\mathbb{E}_{\tilde{\boldsymbol{x}}}\left[\bar{D}^1\boldsymbol{w}^{(2)}(\boldsymbol{w}^{(2)})^T\bar{D}^1\otimes\tilde{\boldsymbol{x}}\tilde{\boldsymbol{x}}^{T}\right],\label{eq:M^(1)}
\end{multline}
  $\bar{D}^1=\E_{W^{(1)}}\left[D^1\right]$,
and the Fisher matrix which appears in the bound following  (\ref{eq:reduced Fisher layered model})
is
\begin{multline}
I(X,\boldsymbol{w}^{(2)})=\sum_{k=1}^{M}\mathbb{E}_{W^{(1)}|\boldsymbol{w}^{(2)},X}\left[\nabla_{W^{(1)}}f_{\boldsymbol{\theta}}(\boldsymbol{x}^{(k)})\nabla_{W^{(1)}}f_{\boldsymbol{\theta}}(\boldsymbol{x}^{(k)})^{T}\right]\\
=\sum_{k=1}^{M}\mathbb{E}_{W^{(1)}|\boldsymbol{w}^{(2)},X}\left[D^{1(k)}\boldsymbol{w}^{(2)}(\boldsymbol{w}^{(2)})^TD^{1(k)}\otimes\boldsymbol{x}^{(k)}(\boldsymbol{x}^{(k)})^T\right]\label{eq:I^(1)}
\end{multline}
where $\boldsymbol{q}^{(k)}=W^{(1)}\boldsymbol{x}^{(k)}/\sqrt{d}$ and $D^{1(k)}=\mathrm{diag}(\sigma'(\boldsymbol{q}^{(k)}))$. Our work is then to evaluate the
asymptotics of the right side of \eqref{eq:B1_ba}. Toward this end,
we first show that we can replace $Q(\boldsymbol{w}^{(2)})$ in the latter by the simpler matrix
\begin{equation}
	\tilde{Q}(\boldsymbol{w}^{(2)})=\frac{\alpha\xi}{N}\boldsymbol{w}^{(2)}(\boldsymbol{w}^{(2)})^T\otimes I_{d}, \label{eq:Mtilde}
\end{equation}
such that the error due to this replacement goes to zero in the limit of $d\rightarrow \infty$. Indeed,  the replacement error reads
\begin{multline}
	\frac{\sigma_{\varepsilon}^{2}}{d}\mathbb{E}\left[\left|\mathrm{Tr}\left(\left(I(X,\boldsymbol{w}^{(2)})+\sigma_{\varepsilon}^{2}\alpha I_{Nd}\right)^{-1}(Q-\tilde{Q})\right)\right|\right]\leq	\frac{1}{\alpha d}\mathbb{E}\left[\| Q-\tilde{Q}\|_{\mathrm{*}}\right]
	\\ \leq \sqrt{\frac{N}{d}}\frac{1}{\alpha}\mathbb{E}\left[\| Q-\tilde{Q}\|_{\mathrm{HS}}\right]\underset{d\rightarrow\infty}{\rightarrow} 0,
\end{multline}
where $\| \cdot \|_{\mathrm{*}}$ denotes the nuclear norm, the first inequality is due to the positive definiteness of
 $I(X,\boldsymbol{w}^{(2)})$ which implies that the operator norm of the inverse is bounded above by $1/\alpha \sigma_{\varepsilon}^2$, the second follows from
the Cauchy-Schwarz inequality, and the limit follows from  Lemma \ref{lem:Mone}.

Introduce
the constants
\begin{equation}
  \label{eq-morectt}
r_x^{(k)}= {\| \boldsymbol{x}^{(k)}\|_{2}}/{\sqrt{d\alpha}},\quad  \eta_{1}^{(k)}=\mathbb{E}_{z}\left[\sigma'(zr_x^{(k)})^{2}\right],\quad \xi^{(k)}=\mathbb{E}_{z}\left[\sigma'(zr_x^{(k)})\right]^{2},
\end{equation} where
$z\sim\mathcal{N}(0,1)$.
We next show that we can also replace the matrix
$I(X,\boldsymbol{w}^{(2)})$ by the matrix $\tilde{I}(X,\boldsymbol{w}^{(2)})$, where
\begin{equation}
	\tilde{I}(X,\boldsymbol{w}^{(2)})=\frac{\xi}{v^2Nd}\boldsymbol{w}^{(2)}(\boldsymbol{w}^{(2)})^T\otimes XX^{T}+G(X,\boldsymbol{w}^{(2)})
=: \hat I(X,\boldsymbol{w}^{(2)})+G(X,\boldsymbol{w}^{(2)}),\label{eq:Itilde}
\end{equation}
with
\begin{equation}
	G(X,\boldsymbol{w}^{(2)}) = \frac{1}{N}\sum_{k}D^{(2)}\otimes\frac{1}{d}c^{(k)}\boldsymbol{x}^{(k)}(\boldsymbol{x}^{(k)})^T,
\end{equation}
 $D^{(2)}_{ij}=\delta_{ij}\left(w_i^{(2)}\right)^2$  and $c^{(k)}=\eta_{1}^{(k)}-\xi^{(k)}$. Set $I_\alpha = I(X,\boldsymbol{w}^{(2)})+\sigma_\eps^2\alpha I_{Nd}$
and $\tilde{I}_\alpha = \tilde{I}(X,\boldsymbol{w}^{(2)})+\sigma_\eps^2\alpha I_{Nd}$. The error due to the replacement of $I$ by $\tilde I$  is then
\begin{eqnarray}
&&\frac1d  \mathbb{E}\left[\left|\mathrm{Tr}\left(\left(\tilde{I}_{\alpha}^{-1}-I_{\alpha}^{-1}\right)\tilde{Q}\right)\right|\right]
=\frac1d\mathbb{E}\left[\left|\mathrm{Tr}\left(\tilde{I}_{\alpha}^{-1}\left(I-\tilde{I}\right)I_{\alpha}^{-1}\tilde{Q}\right)\right|\right]\nonumber\\
	&\leq&\frac{1}{\alpha^2 d}\mathbb{E}\left[\mathrm{Tr}\left(\tilde{Q}\right)\| I-\tilde{I}\|_{\mathrm{op}}\right]\leq\frac{\xi}{\alpha}\sqrt{\mathbb{E}\left[\frac{\| \boldsymbol{w}^{(2)}\|_{2}^{4}}{N^{2}}\right]}
\left|\sqrt{\mathbb{E}\left[\| I-\tilde{I}\|_{\mathrm{op}}^{2}\right]}\right|\nonumber\\
&\leq&\frac{\xi}{\alpha \alpha_2}\left|\sqrt{\mathbb{E}\left[\| I-\tilde{I}\|_{\mathrm{op}}^{2}\right]}\right|;
\end{eqnarray}
 In the first equality  we used the identity
 \begin{equation}
   \label{eq-AB}
   A^{\text{\textminus}1}\text{\textminus}B^{\text{\textminus}1}=A^{\text{\textminus}1}(B\text{\textminus}A)B^{\text{\textminus}1},
 \end{equation}
in the first inequality  we used the positive definiteness  of $\tilde{Q}$ and that   $\| I^{-1}_\alpha\|_{\mathrm{op}}$, $\|\tilde{I}^{-1}_\alpha \|_{\mathrm{op}}\leq \alpha^{-1}$, in the second inequality we used the  Cauchy-Schwarz inequality, and in the last Assumption \ref{subsec:assumptionsiFF1}. An application of Lemma \ref{lem:FisheraveragedW} below then
shows that $\mathbb{E}\left[\|I-\tilde{I}\|_{\mathrm{op}}^{2}\right]\underset{d\rightarrow\infty}{\rightarrow} 0$, and completes the justification of the
replacement of $I$ by $\tilde I$.
We thus obtained that
\begin{equation}
\label{eq-1703}
B^{(1)}=\frac{\sigma_\eps^2}{d} \E \mathrm{Tr} (\tilde I_\alpha^{-1} \tilde Q)+o(1).
\end{equation}

We next show  that we can replace $\tilde I_\alpha$ in  \eqref{eq-1703} by $\hat I_\alpha=\tilde I_\alpha-G= \hat I(X,\boldsymbol{w}^{(2)})+\sigma_\eps^2\alpha I_{Nd}$, i.e. remove $G$ in \eqref{eq:Itilde}, with an $o(1)$ error. Indeed, using the identity  \eqref{eq-AB}
 with $A=\tilde I_\alpha$ and $B=\hat I_\alpha$,  and using that
 $\|\tilde I_\alpha\|_{\mathrm{op}},\|\hat I_\alpha\|_{\mathrm{op}}
 \leq 1/\alpha \sigma_\eps^2$ we obtain, using that the rank of
 $\tilde Q$ is bounded above by $d$, that
\begin{equation}
  \label{eq-1703aa}
  \frac{1}{d} \E \mathrm{Tr}((\tilde I_\alpha-\hat I_\alpha)\tilde Q)
  \leq C \E [\|G\|_{\mathrm{op}} \|\tilde Q\|_{\mathrm{op}}],
\end{equation}
where $C$, here and in the next few lines,
is a constant independent of $d$. By construction,
$\E \|\tilde Q\|_{\mathrm{op}}^2 \leq C$. Also, from
standard properties of Wishart matrices,
$\E \|d^{-1}\sum_k \boldsymbol{x}^{(k)}(\boldsymbol{x}^{(k)})^T\|_{\mathrm{op}}^4\leq C$ while, using that the entries $w_i^{(2)}$ are standard
i.i.d. Gaussians,
$\E \max_{i} (w_i^{(2)})^4\leq C (\log d)^2$  and therefore
$\E \|G\|_{\mathrm{op}}^2 \leq C \log d/N^2$. Therefore, the right hand side
of \eqref{eq-1703aa} is bounded above by $\sqrt{\log d}/N=o(1)$. Altogether,
we obtain that
\begin{equation}
\label{eq-1703bb}
B^{(1)}=\frac{\sigma_\eps^2}{d} \E \mathrm{Tr} (\hat
I_\alpha^{-1} \tilde Q)+o(1).
\end{equation}
Set  $J_{w}=\boldsymbol{w}^{(2)}\otimes I_{d}/\sqrt{N}\in\mathbb{R}^{Nd\times d}$,
and $J_{X,w}=\boldsymbol{w}^{(2)}\otimes X/\sigma_x\sqrt{Nd}\in\mathbb{R}^{Nd\times M}$. Then,
\begin{equation}
  B^{(1)}=\frac{\sigma_{\varepsilon}^{2}\alpha\xi}{d}\mathbb{E}\left[\mathrm{Tr}\left(\left(\alpha\xi J_{X,w}J_{X,w}^{T}+\alpha\sigma_{\varepsilon}^{2}I_{Nd}\right)^{-1}J_{w}J_{w}^{T}\right)\right]+
  o(1).\label{eq:B1_simplifyied}
\end{equation}
By Woodbury's formula, we have that
\begin{multline}
\left(\alpha\xi J_{X,w}J_{X,w}^{T}+\alpha\sigma^2_{\varepsilon}I_{Nd}\right)^{-1}=(\alpha\sigma^2_{\varepsilon})^{-1}I_{Nd}\\-(\alpha\sigma^2_{\varepsilon})^{-2}\alpha\xi J_{X,w}\left(I_{M}+\sigma_{\varepsilon}^{-2}\xi J_{X,w}^{T}J_{X,w}\right)^{-1}J_{X,w}^{T}\label{eq:Woodbury_1}
\end{multline}
Substituting  (\ref{eq:Woodbury_1}) back in (\ref{eq:B1_simplifyied}), we have that:
\begin{multline}
B^{(1)}=o(1)+\frac{\xi}{d}\mathbb{E}_{\boldsymbol{w}^{(2)}}\left[\mathrm{Tr}\left(J_{w}^{T}J_{w}\right)\right]
\\-\frac{\xi^{2}}{\sigma_{\varepsilon}^{2}}\mathbb{E}_{\boldsymbol{w}^{(2)},X}\left[\frac{1}{d}\mathrm{Tr}\left(\left(I_{M}+\frac{\xi}{\sigma_{\varepsilon}^{2}}J_{X,w}^{T}J_{X,w}\right)^{-1}J_{X,w}^{T}J_{w}J_{w}^{T}J_{X,w}\right)\right],\label{eq:B1_simplified2}
\end{multline}
where we define $r_{w}=\| \boldsymbol{w}^{(2)}\|_{2}^{2}/N$.
The first term in the right hand side of \eqref{eq:B1_simplified2} is simply
\begin{equation}
\frac{1}{d}\mathbb{E}_{\boldsymbol{w}^{(2)}}\left[\mathrm{Tr}\left(J_{w}^{T}J_{w}\right)\right]=\frac{1}{d}\mathbb{E}_{\boldsymbol{w}^{(2)}}\left[\mathrm{Tr}\left(\boldsymbol{w}^{(2)}(\boldsymbol{w}^{(2)})^T\otimes I_{d}\right)\right]=\mathbb{E}_{\boldsymbol{w}^{(2)}}r_{w}=\frac{1}{\alpha_{2}}.
\end{equation}
To analyze the second term, we first calculate,
\begin{equation}
J_{X,w}^{T}J_{w}=\left(\frac{1}{\sqrt{Nd}}(\boldsymbol{w}^{(2)})^T\otimes X^{T}\right)\left(\frac{1}{\sqrt{N}}\boldsymbol{w}^{(2)}\otimes I_{d}\right)=r_{w}\frac{1}{\sqrt{d}}X^{T}\label{eq:JXwJw}
\end{equation}
Hence, $J_{X,w}^{T}J_{w}J_{w}^{T}J_{X,w}=r_{w}^{2}X^{T}X/d$.
We also have that
\begin{equation}
J_{X,w}^{T}J_{X,w}=\left(\frac{1}{\sqrt{Nd}}(\boldsymbol{w}^{(2)})^T\otimes X^{T}\right)\left(\frac{1}{\sqrt{Nd}}\boldsymbol{w}^{(2)}\otimes X\right)=r_{w}\frac{1}{d}X^{T}X\label{eq:JXwJXw}
\end{equation}
Using  (\ref{eq:JXwJw}) and  (\ref{eq:JXwJXw}) the bound in the right hand side
of
 (\ref{eq:B1_simplified2}) reads:
\begin{multline}
\frac{\xi}{\alpha_{2}}-\frac{\xi^{2}}{\sigma_{\varepsilon}^{2}}\mathbb{E}_{\boldsymbol{w}^{(2)},X}\left[\frac{1}{d}\mathrm{Tr}\left(\left(I_{M}+\frac{\xi}{\sigma_{\varepsilon}^{2}}r_{w}\frac{1}{d}X^{T}X\right)^{-1}r_{w}^{2}\frac{1}{d}X^{T}X\right)\right]\\
=\frac{\xi}{\alpha_{2}}-\xi^{2}\mathbb{E}_{\boldsymbol{w}^{(2)},X}\left[\frac{1}{d}\sum_{i}\frac{r_{w}^{2}\lambda_{i}(\frac{1}{d}X^{T}X)}{\sigma_{\varepsilon}^{2}+\xi r_{w}\lambda_{i}(\frac{1}{d}X^{T}X)}\right].
\end{multline}
By dominated convergence and the convergence of the empirical
measure of eigenvalues of $X^TX/d$ in the regime $M,d\rightarrow\infty$ such that  $\gamma_0=\lim_{d\to\infty} {d}/M\in(0,\infty)$
to the Marchenko-Pastur law and the
convergence of $r_w$ to $1/\alpha_2$, the last expression converges
(as $d\to\infty$) to
\begin{multline}
  \frac{\xi}{\alpha_{2}}\left(1-\frac{\xi}{\gamma_{0}}\int\frac{s}{\xi s+\alpha_{2}\sigma_{\varepsilon}^{2}}d\rho_{\gamma_{0}^{-1}}(s)\right)\\
=\sigma_{\varepsilon}^{2}\xi\int\frac{1+\xi(1-\gamma_{0}^{-1})s/\sigma_{\varepsilon}^{2}\alpha_{2}}{\xi s+\alpha_{2}\sigma_{\varepsilon}^{2}}d\rho_{\gamma_{0}^{-1}}(s).
\label{eq:finalB1}
\end{multline}
This yields \eqref{eq:B1Tight}.

\paragraph{Evaluation of $B^{(2)}$}
The bound $B^{(2)}$ is derived by taking $\boldsymbol{\theta}_{l}=W^{(1)}$ in Corollary \ref{Cor: Layer-wise bound}. Hence, following
 (\ref{eq:B^{(l)}}), the bound
reads
\begin{equation}
B^{(2)}=\frac{1}{N}\mathbb{E}_{W^{(1)},X}\left[\mathrm{Tr}\left(\left(I(X,W^{(1)})\right)^{-1}\Sigma(W^{(1)})\right)\right],\label{eq:B^(2) init}
\end{equation}
where $X\in\mathbb{R}^{d\times M}$ is the features matrix and $\boldsymbol{x}^{(k)}\in\mathbb{R}^{d\times1}$ are its columns,  the matrix
$\Sigma=\Sigma(W^{(1)})\in\mathbb{R}^{N\times N}$ is
\begin{equation}
\Sigma(W^{(1)})
=\E_{\bm{x}}\left[\sigma({\boldsymbol{q}})\sigma({\boldsymbol{q}})^{T}\right],\label{eq:M1}
\end{equation}
and the Fisher matrix $I$ is
\begin{equation}
I(X,W^{(1)})=\frac{1}{\sigma_{\varepsilon}^{2}N}\sum_{k=1}^{M}\sigma(\boldsymbol{q}^{(k)})\sigma(\boldsymbol{q}^{(k)})^{T}+\alpha_{2}I_{N},\label{eq:I1}
\end{equation}
where $\boldsymbol{q}^{(k)}=W^{(1)}\boldsymbol{x}^{(k)}/\sqrt{d}$. Substituting
the matrices (\ref{eq:M1}) and  (\ref{eq:I1}) in  (\ref{eq:B^(2) init}),
such that $X^{(1)}=\sigma(W^{(1)}X/\sqrt{d})\in\mathbb{R}^{N\times M}$,
we obtain that
\begin{multline}
B^{(2)}
=\frac{\sigma_{\varepsilon}^{2}}{N}\mathbb{E}_{W^{(1)},X}\left[\mathrm{Tr}\left(\left(\frac{1}{N}X^{(1)}(X^{(1)})^T+\sigma_{\varepsilon}^{2}\alpha_{2}I_{N}\right)^{-1}\Sigma\right)\right]\\
=\frac{\sigma_{\varepsilon}^{2}}{d}\mathbb{E}_{W^{(1)},X}\left[\mathrm{Tr}\left(\left(\frac{1}{d}X^{(1)}(X^{(1)})^T+\alpha_d I_{N}\right)^{-1}\Sigma\right)\right],\label{eq:Bound expression}
\end{multline}
where we define $\alpha_d = {N\sigma_{\varepsilon}^{2}\alpha_{2}}/{d}$. To calculate the expression above, we utilize the framework in
\cite{mei2019generalization} adapted to our settings. We first note
that
by Lemma \ref{lem:population kernel}, the operator norm difference between
the matrix $\Sigma$
and the matrix
\begin{equation}
	\tilde{\Sigma}=I_{N}(\eta_{0}-\xi)+\alpha\xi\frac{1}{d}W^{(1)}(W^{(1)})^T+\frac{a}{d}\left(1_{N}1_{N}^{T}-I_{N}\right)\label{eq:sigma_tilde}
\end{equation}
(with $a\in\mathbb{R}$ an appropriate constant) is $o(1)$.
The error due to this replacement is bounded
by $\| \Sigma-\tilde{\Sigma}\|_{\mathrm{op}}/\alpha_{2}$, hence the expected error goes to zero.
	Next, note that
	the contribution of the third term in  (\ref{eq:sigma_tilde}) is $o(1)$ since
	\[
	\frac{\sigma_{\varepsilon}^{2}}{d^{2}}\left|\mathrm{Tr}\left(\left(\frac{1}{d}X^{(1)}(X^{(1)})^T+\alpha_dI_{N}\right)^{-1}(1_{N}1_{N}^{T}-I_N)\right)\right|\le\frac{1}{\alpha_{2}d}.
	\]
Therefore, to calculate (\ref{eq:Bound expression}) we need
to evaluate the two terms

\[
T_{1}=\mathbb{E}_{W^{(1)},X}\left[\frac{1}{d}\mathrm{Tr}\left(\left(\frac{1}{d}X^{(1)}(X^{(1)})^T+\alpha_dI_{N}\right)^{-1}\right)\right]
\]
and
\[
T_{2}=\mathbb{E}_{W^{(1)},X}\left[\frac{1}{d}\mathrm{Tr}\left(\left(\frac{1}{d}X^{(1)}(X^{(1)})^T+\alpha_dI_{N}\right)^{-1}W^{(1)}(W^{(1)})^T\right)\right].
\]
Adopting the notation of  \cite{mei2019generalization}, we consider
the following matrix $A(\boldsymbol{s})\in\mathbb{R}^{r\times r}$,
where $r=M+N$:
\[
A(\boldsymbol{s})=\left[\begin{array}{cc}
s_{1}I_{N}+s_{2}Q & (X^{(1)})^T\\
X^{(1)} & 0_{M}
\end{array}\right].
\]
For our purposes, we take $\boldsymbol{s}=(s_{1},s_{2})\in\mathbb{R}^{2}$,
and $Q_{ij}=\frac{1}{d}\sum_{k}W_{ik}^{(1)}W_{jk}^{(1)}$.
For $t\in\mathbb{C_{+}},$ we introduce the log-determinant
\[
G_d(t,\boldsymbol{s})=\frac{1}{d}\sum_{i=1}^{r}\mathrm{log}(\lambda_{i}(A)-t)=\frac{1}{d}\mathrm{log}\left(\mathrm{det}\left(A(\boldsymbol{s})-t I_{r}\right)\right)+i2\pi k(\boldsymbol{s},t),
\]
where $k(\boldsymbol{s},t)\in\mathbb{N}$, and
the Stieltjes transform of the empirical measure
of eigenvalues of $A(\boldsymbol{s})$
\begin{equation}
M_d(t,\boldsymbol{s})=\frac{1}{d}\sum_{i}(\lambda_{i}(A)-t)^{-1}=-\frac{d}{dt}G_d(t,\boldsymbol{s}).\label{eq:M_d}
\end{equation}
Using the definition of the log-determinant and applying simple algebraic manipulations, see  \cite[Appendix B]{mei2019generalization} for more details, we then have
\begin{multline*}
\partial_{s_{1}}G_d(iu,\boldsymbol{0})=\frac{1}{d}\mathrm{Tr}\left(\left(A(\boldsymbol{0})-iuI_{r}\right)^{-1}\partial_{s_{1}}A(\boldsymbol{0})\right)\\
=\frac{1}{d}\mathrm{Tr}\left(\left(A(\boldsymbol{0})-iuI_{r}\right)^{-1}\partial_{s_{1}}A(\boldsymbol{0})\right)\\
=\frac{1}{d}\mathrm{Tr}\left(\left(-iuI_{N}+\left(iu\right)^{-1}X^{(1)}(X^{(1)})^T\right)^{-1}\right)=iu\frac{1}{d}\mathrm{Tr}\left(\left(u^{2}I_{N}+X^{(1)}(X^{(1)})^T\right)^{-1}\right)
\end{multline*}
and
\begin{multline*}
\partial_{s_{2}}G_d(iu,\boldsymbol{0})=\frac{1}{d}\mathrm{Tr}\left(\left(A(\boldsymbol{0})-t I_{r}\right)^{-1}\partial_{s_{2}}A(\boldsymbol{0})\right)\\
=\frac{1}{d}\mathrm{Tr}\left\{ \left(-iuI_{N}-iu{}^{-1}X^{(1)}(X^{(1)})^T\right)^{-1}Q\right\} \\
=iu\frac{1}{d}\mathrm{Tr}\left\{ \left(u^{2}I_{N}+X^{(1)}(X^{(1)})^T\right)^{-1}Q\right\}.
\end{multline*}
Therefore, substituting in  (\ref{eq:Bound expression}),
\begin{multline}
\sigma_{\varepsilon}^{2}\left[\alpha\xi T_{2}+(\eta_{0}-\xi)T_{1}\right]\\
=-\frac{\sigma_{\varepsilon}^{2}}{u_c}\mathbb{E}_{W^{(1)},X}\left[\alpha\xi i\partial_{s_{2}}G_d(iu_c,\boldsymbol{0})+(\eta_{0}-\xi)i\partial_{s_{1}}G_d(iu_c,\boldsymbol{0})\right],\label{eq:B vs g}
\end{multline}
where $u_c=\sigma_{\varepsilon}\sqrt{{\alpha_2}{\beta_{1}}}$.

Introduce the (deterministic) functions
\[g(t,\boldsymbol{s})=\varXi(t,z_{1},z_{2},\boldsymbol{s})|_{(z_{1},z_{2})=(m_{1}(t,\boldsymbol{s}),m_{2}(t,\boldsymbol{s}))}
\]
and
\begin{multline*}
\varXi(t,z_{1},z_{2},\boldsymbol{s})=\mathrm{log}\left[(s_{2}z_{1}+1)-\xi z_{1}z_{2}\right]-(\eta_{0}-\xi)z_{1}z_{2}+s_{1}z_{1}\\
-{\beta_{1}}\mathrm{log}\left(\frac{z_{1}}{\beta_{1}}\right)-\frac{1}{\gamma_{0}}\mathrm{log}\left(\gamma_{0}{z_{2}}\right)-\xi(z_{1}+z_{2})-{\beta_{1}}-\frac{1}{\gamma_{0}},
\end{multline*}
where the functions
$m_1(\cdot;\bm{s}), m_2(\cdot;\bm{s}): \mathbb{C}_+ \rightarrow \mathbb{C}_+$
are the unique analytic solutions in $\mathbb{C}_+$
 with growth $1/\xi$ at infinity,
of the following equations:
\begin{equation}
m_{1}={\beta_{1}}\left(-t+s_1-(\eta_{0}-\xi)m_{2}+\frac{s_2-\xi m_{2}}{1+s_2m_1-\xi m_{1}m_{2}}\right)^{-1} \label{eq:m1Fixed}
\end{equation}

\begin{equation}
m_{2}=\frac{1}{\gamma_{0}}\left(-t-(\eta_{0}-\xi )m_{1}-\frac{\xi m_{1}}{1+s_2m_1-\xi m_{1}m_{2}}\right)^{-1}, \label{eq:m2Fixed}
\end{equation}
where we wrote for brevity $m_1=m_1(t;\bm{s})$ and
$m_2=m_2(t;\bm{s})$. (That the solutions of \eqref{eq:m1Fixed}
and \eqref{eq:m2Fixed} are unique and analytic
in a neighborhood of $\infty$ follows from the
implicit function theorem, and the uniqueness in $\mathbb{C}_+$ follows from analytic continuation.)
By
\cite[Proposition 8.4]{mei2019generalization},
we have that for any fixed $t\in \mathbb{C}_+$,
\begin{equation}
  \label{eq-L1L2}
  \E|G_d(t,\bm{s})-g(t,\boldsymbol{s})|+
\E\|\nabla_{\bm{s}}G_d(t,\bm{0})-\nabla_{\bm{s}}g(t,\bm{0})\|^2\to_{d\to\infty} 0.\end{equation}
(We remark that by \cite[Proposition 8.3]{mei2019generalization},
the Stieltjes transform $M_d(t,\bm{s})$ of \eqref{eq:M_d} converges uniformly in compacts,  in probability, to
$m_1+m_2$ as in \eqref{eq:m1Fixed} and \eqref{eq:m2Fixed}.)

Summarizing, we have that
\begin{equation}
B^{(2)}=-\frac{\sigma_{\varepsilon}^{2}}{u_c}\mathbb{E}_{W^{(1)},X}\left[\xi i\partial_{s_{2}}g(iu_c,\boldsymbol{0})+(\eta_{0}-\xi)i\partial_{s_{1}}g(iu_c,\boldsymbol{0})\right] +o(1),
\label{eq-semifinal}
\end{equation}
so that it only remains to evaluate the derivatives of $g$ appearing in
\eqref{eq-semifinal}.
Toward this end, we note, following \cite{mei2019generalization}, that
the fixed point equations (\ref{eq:m1Fixed}) and (\ref{eq:m2Fixed})
imply that $(m_1;m_2)$ is a stationary point of the function
$\Xi(t,\cdot,\cdot,\bm{s})$, that is
$\partial_{z_i} \Xi(t,z_1,z_2,\bm(s))|_{z_1=m_1,z_2=m_2}=0$ for $i=1,2$.
This simplifies the calculation of derivatives with respect to $\bm{s}$ and yields, by
taking the derivative of $\varXi$ with respect to $\boldsymbol{s}$, that
\[
\partial_{s_{2}}g(t,\boldsymbol{0})
=\frac{m_{1}}{1
	-\xi m_{1}m_{2}},
\]
and
\[
\partial_{s_{1}}g(t,\boldsymbol{0})
=m_{1}.
\]
Specializing \eqref{eq:m1Fixed} and \eqref{eq:m2Fixed} to $\bm{s}=\bm{0}$
and using, with a slight abuse of notation, $m_i=m_i(t,\bm{0})$, we have
thus obtained that
\[
m_{1}={\beta_{1}}\left(-t-\left(\eta_{0}-\xi\right)m_{2}+\frac{\xi m_{2}}{\xi m_{1}m_{2}-1}\right)^{-1}
\]
and
\[
m_{2}=\frac{1}{\gamma_{0}}\left(-t-\left(\eta_{0}-\xi\right)m_{1}+\frac{\xi m_{1}}{\xi m_{1}m_{2}-1}\right)^{-1}.
\]
and, from \eqref{eq-semifinal} with
$t_c=iu_c$,
\begin{equation}
B^{(2)}=-\frac{\sigma_{\varepsilon}}{\sqrt{\alpha_2\beta_{1}}}im_{1}(t_c)\left[\frac{\xi}{1-\xi m_{1}(t_c)m_{2}(t_c)}+\eta_{0}-\xi\right]+o(1).\label{eq:B_2_final}
\end{equation}
Note that, by definition, for $t=ib$ when $b>0$, $m_{1}(ib)$ and $m_{2}(ib)$
are purely imaginary and $\Im m_{1}(ib);\Im m_{2}(ib)>0$, hence
the expression in (\ref{eq:B_2_final}) is real valued.
The theorem is obtained by making the substitutions $m_i\to -im_i$.
\qed

We now provide the proof of lemmas used above. Recall the definition of the matrices $Q$ and $\tilde Q$, see \eqref{eq:M^(1)} and \eqref{eq:Mtilde}.
\begin{lemma}
	\label{lem:Mone}Let Assumptions \ref{subsec:assumptions},  \ref{subsec:assumptionsiFF1} and \ref{subsec:assumptionsiFF2} hold.
		Then, in the regime $N,d\rightarrow \infty$ such that
		%{\clr $\gamma_0=\lim_{d\to\infty} {d}/M\in(0,\infty)$, I don't think this is needed} and
		$\beta_1=\lim_{d\to\infty} N/d$,
	it holds that $\mathbb{E}\left[\| Q-\tilde{Q}\|_{\mathrm{HS}}\right]\underset{d\rightarrow\infty}{\rightarrow}0$
\end{lemma}

\begin{proof}
Recall the definition of the matrix
	\begin{equation}
	Q=\frac{1}{N}\mathbb{E}_{\tilde{\boldsymbol{x}}}\left[\bar{D}^1\boldsymbol{w}^{(2)}(\boldsymbol{w}^{(2)})^T\bar{D}^1\otimes\tilde{\boldsymbol{x}}\tilde{\boldsymbol{x}}^{T}\right],
	\end{equation}
	where
 $\bar{D}^{1}=\E_{W^{(1)}}\left[\mathrm{diag}\left(\sigma'(\boldsymbol{q}_{w})\right)\right]$,
	and $\boldsymbol{q}_{w}=W^{(1)}\boldsymbol{x}/\sqrt{d}$.
Typical elements in the matrix $Q$  are of the form
	\begin{equation}
	Q_{(i_1,i_2,i_1',i_2')}=\frac{1}{N}w^{(2)}_{i_{2}}w^{(2)}_{i'_{2}}\mathbb{E}_{\boldsymbol{x}}\left[\mathbb{E}_{W^{(1)}}\sigma'(q_{wi_{2}})\mathbb{E}_{W^{(1)}}\sigma'(q_{wi'_{2}})x_{i'_{1}}x_{i_{1}}\right]\label{eq:element_expen_sw_su}
	\end{equation}
	Since by assumption $W^{(1)}$ is a Gaussian matrix, the variables $q_{wi}\sim\mathcal{N}(0,r_x^{2})$
	for all $i$, where $r_x ={\| \boldsymbol{x}\|_{2}}/{\sqrt{\alpha d}} $. Therefore, $\mathbb{E}_{W^{(1)}}\left[\sigma'(q_{wi_{2}})\right]=\mathbb{E}_{z}\left[\sigma'(zr_x)\right],$ with
	$z\sim\mathcal{N}(0,1)$. Write for brevity $g({\| \boldsymbol{x}\|_{2}^{2}}/{d})=\mathbb{E}_{z}\left[\sigma'(zr_x)\right]^{2}$
	and $\tilde{z}_{i_1i_1'}=\sum_{k\neq i_{1},i'_{1}}x_{k}^{2}/d$.
	Applying a Taylor expansion around $\tilde{z}_{x}$, we obtain
	\begin{multline}
	\mathbb{E}\left[x_{i'_{1}}x_{i_{1}}g(\frac{\| \boldsymbol{x}\| _{2}^{2}}{d})\right]
	=\delta_{i_{1}i'_{1}}\sigma_{x}^{2}\mathbb{E}\left[g(\tilde{z}_{i_1i_1})\right]+\delta_{i_{1}i'_{1}}\frac{1}{d}6\sigma_{x}^{4}\mathbb{E}\left[g'(\tilde{z}_{i_1i_1})\right]
	+\frac{1}{d^2}\mathcal{E}_{i_1i_1'}
	\label{eq:expan xxNorm}
	\end{multline}
	where $\mathcal{E}_{i_1i_1'}=\mathbb{E}\left[x_{i'_{1}}x_{i_{1}}(x_{i'_{1}}^2+x_{i'_{1}}^2)^2g''(\eta_{i_1i_1'})\right]/2$ and  $\eta_{i_1i_1'}$ is a random point between $\tilde{z}_{i_1i_1'}$ and
	$\| \boldsymbol{x}\|_{2}^{2}/d$. Using Chebyshev's inequality,
	$\| \boldsymbol{x}\|_{2}^{2}/d\rightarrow\sigma_{x}^{2}$ in probability, and since by assumption $\sigma'$ is bounded, the bounded convergence
	theorem  yields that
	$\mathbb{E}\left[g(\tilde{z}_{i_1i_1})\right]\rightarrow\alpha\xi$. Since $g'$ is bounded by assumption, $\mathbb{E}\left[g'(\tilde{z}_{i_1i_1})\right]$ is also some bounded constant which do not depend on the index $i_1$.
	Therefore,
	substituting back in  (\ref{eq:element_expen_sw_su}),
	the first term yields the matrix
	$\tilde{Q}$.
We now calculate the error $	\mathbb{E}\left[\| Q-\tilde{Q}\|_\mathrm{HS}\right]$ due to the second and third terms in  (\ref{eq:expan xxNorm}).
The error from the second term is $o(1)$ since
	\begin{equation}
			\frac{1}{dN}\E\left[\|I_d\otimes\bm{w}^{(2)}(\bm{w}^{(2)})^T\|_{\mathrm{HS}}\right]
			=\frac{1}{\sqrt{d}N}\mathbb{E}\left[\|\bm{w}^{(2)}\textbf{}(\bm{w}^{(2)})^T\|_\mathrm{HS}\right]\leq\frac{1}{\sqrt{d}\alpha_2}.
	\end{equation}
	The error due to the third term is
	also $o(1)$ since
	\begin{equation}
	\frac{1}{d^2N}\mathbb{E}\left[\|\bm{w}^{(2)}\textbf{}(\bm{w}^{(2)})^T\|_\mathrm{HS}\right]\|\mathcal{E}\|_\mathrm{HS}
	\leq\frac{\|\mathcal{E}\|_\mathrm{HS}}{d^2}\leq \frac{c}{d},
	\end{equation}
	since by the Cauchy-Schwarz inequality, the
	boundedness of $g''$,
	and the fact that
	$\boldsymbol{x}$ has finite moments by assumption, we have that $\|\mathcal{E}\|_\mathrm{HS}\leq dc$ where $c$ is a constant independent of $d$. Therefore the error due to the forth term is $o(1)$. This completes the proof.
\end{proof}
Recall the constants in \eqref{eq-morectt}, and the matrices $I$, $\tilde I$,
see \eqref{eq:Itilde} and \eqref{eq:I^(1)}.
\begin{lemma}\label{lem:FisheraveragedW}
	Let Assumptions \ref{subsec:assumptions},  \ref{subsec:assumptionsiFF1}, and \ref{subsec:assumptionsiFF2} hold.
	Then, in the regime $N,d,M\rightarrow \infty$ such that $\gamma_0=\lim_{d\to\infty} {d}/M\in(0,\infty)$ and $\beta_1=\lim_{d\to\infty} N/d$,
	it holds that $\E\left[\| I-\tilde{I}\|_{\mathrm{op}}^{2}\right]\underset{d\rightarrow\infty}{\rightarrow}0$.
\end{lemma}

\begin{proof}
 Write
 \[I(X,\boldsymbol{w}^{(2)})=\sum_{k}\mathbb{E}_{W^{(1)}|\boldsymbol{w}^{(2)},X}\left[\boldsymbol{s}_{w}^{(k)}(\boldsymbol{s}_{w}^{(k)})^T\otimes\frac{1}{d}\boldsymbol{x}^{(k)}(\boldsymbol{x}^{(k)})^T\right]
	\]
	where $\boldsymbol{s}_{w}^{(k)}=D^{(k)}\boldsymbol{w}^{(2)}/\sqrt{N}$
	with
	$D^{(k)}=\mathrm{diag}\left(\sigma'(\boldsymbol{q}_{w}^{(k)})\right)$,
	and $\boldsymbol{q}_{w}^{(k)}=W^{(1)}\boldsymbol{x}^{(k)}/\sqrt{d}$.

	We will first analyze the matrix
	\begin{multline}
	\mathbb{E}_{W^{(1)}|\boldsymbol{w}^{(2)},X}\left[\boldsymbol{s}_{w}^{(k)}(\boldsymbol{s}_{w}^{(k)})^T\right]\\=\frac{1}{N}\mathbb{E}_{W^{(1)}|\boldsymbol{w}^{(2)},X}\left[\mathrm{diag}\left(\sigma'(\boldsymbol{q}_{w}^{(k)})\right)\boldsymbol{w}^{(2)}(\boldsymbol{w}^{(2)})^T\mathrm{diag}\left(\sigma'(\boldsymbol{q}_{w}^{(k)})\right)\right]\label{eq:SS^T}.
	\end{multline}
	We note that $\mathbb{E}_{W^{(1)}|X}\left[{q}_{i,w}^{(k)}{q}_{jw}^{(k)}\right]=\delta_{ij}(r_{x}^{(k)})^{2}$.
	We now calculate, using that $q_{w,i}^{(k)}\sim\mathcal{N}(0,\left(r_{x}^{(k)}\right)^{2})$
	for all $i$,
	\begin{multline}
	\mathbb{E}_{W^{(1)}|X}\left[\sigma'(q_{w,i}^{(k)})\sigma'(q_{w,j}^{(k)})\right]=\delta_{ij}\mathbb{E}_{W^{(1)}|X}\left[\sigma'(q_{w,i}^{(k)})^{2}\right]+(1-\delta_{ij})\mathbb{E}_{W^{(1)}|,X}\left[\sigma'(q_{w,i}^{(k)})\right]^{2}\\
	=\delta_{ij}\mathbb{E}_{z}\left[\sigma'(zr_{x}^{(k)})^{2}\right]+(1-\delta_{ij})\mathbb{E}_{z}\left[\sigma'(zr_{x}^{(k)})\right]^{2}\\
	=\delta_{ij}\left(\eta_{1}^{(k)}-\xi^{(k)}\right)+\xi^{(k)}.\label{eq:W_correlation}
	\end{multline}
	Substituting  in (\ref{eq:W_correlation})
	and using the definition of $\tilde I$,
	we obtain that 	
	\begin{multline}
	\mathbb{E}\left[\| I-\tilde{I}\|_{\mathrm{op}}^{2}\right]=
	\mathbb{E}\left[\| \sum_{k}\left[\frac{1}{N}\boldsymbol{w}^{(2)}(\boldsymbol{w}^{(2)})^T\Delta \theta^{(k)}\right]\otimes\frac{1}{d}\boldsymbol{x}^{(k)}(\boldsymbol{x}^{(k)})^T\| _{\mathrm{op}}^{2}\right]\\
\leq N
\sum_{k}\mathbb{E}\left[\left\Vert \left(\frac{1}{N}\boldsymbol{w}^{(2)}(\boldsymbol{w}^{(2)})^T\Delta \theta^{(k)}\right)\otimes\frac{1}{d}\boldsymbol{x}^{(k)}(\boldsymbol{x}^{(k)})^T\right\Vert _{\mathrm{op}}^{2}\right]\\
	=\frac{1}{Nd^{2}}\sum_{k}\mathbb{E}\left[\left(\Delta \theta^{(k)}\right)^{2}\left\Vert \boldsymbol{x}^{(k)}\right\Vert _{\mathrm{2}}^{2}\right]\mathbb{E}\left[\left\Vert \boldsymbol{w}\right\Vert _{\mathrm{2}}^{2}\right]\\
	\leq\frac{1}{d^{2}\alpha_{2}}\sum_{k}\sqrt{\mathbb{E}\left[\left(\Delta \theta^{(k)}\right)^{4}\right]}\sqrt{\mathbb{E}\left[\left\Vert \boldsymbol{x}^{(k)}\right\Vert _{\mathrm{2}}^{4}\right]}
	\label{eq:Error I-I},
	\end{multline}
	where we set $\Delta \theta^{(k)}=\xi^{(k)}-{\alpha\xi}/{\sigma_{x}^{2}}$, and used
	in the first inequality the triangle inequality,
	in the second
	equality that
	$\| A\otimes B\|_{\mathrm{op}}=\| A\|_{\mathrm{op}}\|
	B\|_{\mathrm{op}}$, and the Cauchy-Schwarz inequality in the
	last line.
	Note that $\mathbb{E}(\Delta \theta^{(k)})^{4}=o(1)$.
We therefore obtain that the
right hand side of \eqref{eq:Error I-I} is
$o(1)\cdot (M/d)=o(1)$, which completes the proof of the lemma.
\end{proof}

For the next lemma, recall the matrices $\Sigma,\tilde \Sigma$, see
\eqref{eq:M1} and \eqref{eq:sigma_tilde}.
\begin{lemma}
	\label{lem:population kernel}
	Let Assumptions \ref{subsec:assumptions}, \ref{subsec:assumptionsiFF1} and \ref{subsec:assumptionsiFF2} hold. 	Then, in the regime $N,d,M\rightarrow \infty$ such that $\gamma_0=\lim_{d\to\infty} {d}/M\in(0,\infty)$ and $\beta_1=\lim_{d\to\infty} N/d$,
	it holds that
	$\| \Sigma-\tilde{\Sigma}\|_{\mathrm{op}}\underset{d\rightarrow\infty}{\rightarrow0}$
	in probability.
\end{lemma}
\begin{remark}
	A similar statement for spherical weights and data was proved in
	\cite[Proposition 3]{ghorbani2019linearized}, see also
	\cite[Lemma C.7]{mei2019generalization}.
\end{remark}
\begin{proof}
 We rewrite $\Sigma,\tilde \Sigma$ as
 \begin{equation}
		\Sigma=\mathbb{E}_{\boldsymbol{x}}\left[\sigma(\bm{q})\sigma(\bm{q})^{T}\right],\label{eq:SigmaSigma}
	\end{equation}
	where $\bm{q}=W^{(1)}\bm{x}/\sqrt{d}$, and
	\begin{equation}
	\tilde{\Sigma}=\frac{\alpha\xi}{d}W^{(1)}(W^{(1)})^T+\left(\eta_{0}-\xi\right)I_{N_1}
	+c\frac{1}{d}(1_{N_1}1_{N_1}^{T}-I_{N_1}),
	\end{equation}
	where $c$ is an appropriate  constant.
	We calculate the elements in the matrix $\Sigma$ of
	\eqref{eq:SigmaSigma},
	\[
	\Sigma_{ij}=\mathbb{E}_{x}\left[\sigma(q_{i})\sigma(q_{j})\right],
	\]
	where $q_{i}=\sum_{k}W_{ik}^{(1)}x_{k}/\sqrt{d}$
	and $\mathbb{E}_{\boldsymbol{x}}q_{i}=0$,
	since $\bm{x}$ is centered by assumption.
	We now apply Lemma \ref{lem: Gaussian expenssion} for
	each Gaussian bi-vector $(q_{i},q_{j})$,
	with the functions $f=g=\sigma$, such that
	$\varepsilon_{ij}=\mathbb{E}_{x}\left[q_{i}q_{j}\right]
	=\sigma_{x}^{2}\sum_{k}W_{ik}^{(1)}W_{jk}^{(1)}/d$,
	and $v_{i}^{2}=\sigma_{x}^{2}\sum_{k}\left(W_{ik}^{(1)}\right)^{2}/d$.
	Therefore,
	\begin{eqnarray}
	\Sigma_{ij}&=&\theta_{1,i}\theta_{1,j}\frac{\sigma_x^2}{d}\sum_{k}W_{ik}^{(1)}W_{jk}^{(1)}\nonumber\\
	&&+\left(\eta_{0,i}-\theta_{1,i}^{2}\right)\delta_{ij}
	+\sigma_x^4\sum_{l=1}^{L}a_{l}(v_{i})R_{ij}a_{l}(v_{j})+O(\varepsilon_{ij}^{5/2}){\bf 1}_{i\neq j}\label{eq:sigma_corr}
	\end{eqnarray}
	where for $i,j\in [N_1]$,
	we set
	$R_{ij}=\left(\sum_{k}W_{ik}^{(1)}W_{jk}^{(1)}/d\right)^{2}$ if $i\neq j$,
	$R_{ii}=0$,
	and the implied constant in the $O(\cdot)$ term
	in \eqref{eq:sigma_corr} is uniform in $i,j$.
	We also set $\theta_{1,i}=\mathbb{E}\left[\sigma'(v_{i}z)\right]$,
	$\eta_{0,i}=\mathbb{E}\left[\sigma^{2}(v_{i}z)\right]-\mathbb{E}\left[\sigma(v_{i}z)\right]^2$.
	
	First note that, by Chebyshev's inequality applied on the square of
	the norm of each row of the matrix $W^{(1)}$,
	there exists $t_0=t_0(\sigma_x^2)$ such that
	for any $0<t<t_0$ and all $i$,
	$P(\left|v_i-v\right|\geq t)\leq 2e^{-dt^2/8}$.
	Taking a union bound we have that
	$P(\underset{i}{\max}\left|v_i-v\right|>t)\leq
	2N_1 e^{-dt^2/8}$. Taking  $t=\sqrt{\log N_1}/d^{1/2-\delta}=o(1)$,
	we then have by the Borel--Cantelli lemma that for any $\delta>0$,
	and all large $N$,
\begin{equation}
		\underset{i}{\max}\left|v_i-v\right|< (\log N_1)^{1/2}d^{-1/2+\delta}\,\,\, a.s. \label{eq:norm_w}
\end{equation}

By assumption $\sigma$ and $\sigma'$ are bounded continuous functions, and together with
\eqref{eq:norm_w} we obtain that the
diagonal matrix $(\eta_{0,i}-\theta_{1i}^2)\delta_{ij}$ converges in
operator norm to  $(\eta_{0}-\xi)I_{N_1}$, in probability.
Similarly, using that the matrix  $W^{(1)}(W^{(1)})^T/d$ is a Wishart
matrix and therefore posssesses an operator norm
which is bounded in probability uniformly in $N$, we obtain that
the operator norm of the difference
$(\alpha \xi-\theta_{1,i}\theta_{1,j} \sigma_x^2)
W^{(1)} (W^{(1)})^T/d$
converges to $0$ in probability as $d\to\infty$. It thus remains to handle the
matrices composed of the third terms in \eqref{eq:sigma_corr}, and of
the error terms $\eps_{ij}^{5/2}$.

We now control the third term
	$\sum_{l=1}^{L}a_{l}(v_{i})R_{ij}a_{l}(v_{j})$.
Using the analysis in \cite{el2010spectrum}, we note that the matrix
$R$ is precisely
the matrix $W$ in the notation of
	\cite[Proof of Theorem 2.1]{el2010spectrum},
	hence we have that (see pages 12-20 for the proof details)
	\begin{equation}
	\| R-\frac{1}{d\alpha^{2}}(1_{N_1}1_{N_1}^{T}-I_{N_1})\|_{\mathrm{op}}\rightarrow0\label{eq:R_limit}
	\end{equation}
	in probability. Since the $a_{l}(v_i)$ are continuous and bounded functions by Lemma \ref{lem: Gaussian expenssion}, hence following  (\ref{eq:norm_w}) $a_{l}(v_i)$ is uniformly bounded for all $l$. We now trivially obtain, using the
	boundedness in norm of $R$ in probability, that
	with the matrix $A$ such that $A_{l,i}=a_l(v_i)\delta_{ij}$
	\begin{equation}
	\| \sum_lA_lRA_l-\sum_la_l(v) R\|_{\mathrm{op}}\rightarrow0,
	\quad \mathrm{in \ probability}.
	\end{equation}

	We finally turn to the
	matrix composed of the error terms in
	(\ref{eq:sigma_corr}), that is,
	we bound the operator norm of the symmetric matrix $B$ with entries
	$B_{ij}=\varepsilon_{ij}^{5/2}=R_{ij}\varepsilon_{ij}^{1/2}$. By
	a Chebycheff inequality as above, we obtain that for any $\delta>0$,
	$\max_{i\neq j} |\epsilon_{ij}|<d^{-1/2+\delta}$ in probability
	for all large $d$. Therefore, by Gershgorin's circle theorem,
	the operator norm of $B$ is bounded above by $d\cdot d^{-5/4+ 5\delta/2}=o(1)$, in probability, if $\delta$ is chosen small enough.
\end{proof}

\subsection{Proof of Theorem \ref{Thm:CRLB_unbiased}}
\label{Sec:CRLB_unbiased_proof}
%\subsubsection{Proof of Theorem \ref{Thm:CRLB_unbiased}}
We will use Theorem \ref{thm:CRLB}. Since the estimator is assumed to be unbiased, we have in the notation of Theorem \ref{thm:CRLB}, that $\psi_{\boldsymbol{\theta}}(\tilde{\boldsymbol{x}})=f_{\boldsymbol{\theta}}(\tilde{\boldsymbol{x}})$, hence $\nabla_{\boldsymbol{\theta}}\psi_{\boldsymbol{\theta}}(\tilde{\boldsymbol{x}})=\nabla_{\boldsymbol{\theta}}f_{\boldsymbol{\theta}}(\tilde{\boldsymbol{x}})$.
The Fisher information evaluated on the model (\ref{eq:model})
given a set of $M$ i.i.d. measurements $\boldsymbol{z}^{(i)}$ drawn from the  probability
$p(Z|\boldsymbol{\theta})=\prod_{i}p(\boldsymbol{z}^{(i)}|\boldsymbol{\theta})$,
is then $I(\boldsymbol{\theta})= MA(\boldsymbol{\theta})/\sigma_{\varepsilon}^{2}$, with $A=A(\boldsymbol{\theta})$ as in \eqref{eq:Fisher_model}. Write the spectral resolution of $A$
as
$A=U\Lambda U^{T}$, where $U\in \mathbb{R}^{P\times P}$ is an orthogonal matrix, and $\Lambda$ is diagonal with
entries the eigenvalues of $A$, arranged in decreasing order. Then $\Lambda=\begin{bmatrix} \Lambda_k&0\\0&0
\end{bmatrix}$, where $k=\mbox{\rm rank}(A)$ and $\Lambda_k$ is invertible. Denote by $\Lambda^{-1}$ the pseudo-inverse of $\Lambda$, that is
$\Lambda^{-1}=\begin{bmatrix}\Lambda_k^{-1}&0\\0&0\end{bmatrix}$; we define in the same way $\Lambda^{-1/2}$.
Then, with $J_{\boldsymbol{\theta}}(\tilde{\boldsymbol{x}}):=\nabla_{\boldsymbol{\theta}}f_{\boldsymbol{\theta}}(\tilde{\boldsymbol{x}})$, the CR bound reads
\begin{multline}
B^{\mathrm{ub}}=\frac{\sigma_{\varepsilon}^{2}}{M}\mathbb{E}_{\boldsymbol{\theta},\tilde{\boldsymbol{x}}}\left[\max_{\boldsymbol{a}\in\mathbb{R}^{P}}\frac{\boldsymbol{a}^{T}\nabla_{\boldsymbol{\theta}}f_{\boldsymbol{\theta}}(\tilde{\boldsymbol{x}})\left(\nabla_{\boldsymbol{\theta}}f_{\boldsymbol{\theta}}(\tilde{\boldsymbol{x}})\right)^{T}\boldsymbol{a}}{\boldsymbol{a}^{T}A(\boldsymbol{\theta})\boldsymbol{a}}\right]\\
=\frac{\sigma_{\varepsilon}^{2}}{M}\mathbb{E}_{\boldsymbol{\theta},\tilde{\boldsymbol{x}}}\left[\max_{\boldsymbol{b}\in\mathbb{R}^{P}}\frac{\boldsymbol{b}^{T}U^{T}J_{\boldsymbol{\theta}}(\tilde{\boldsymbol{x}})J_{\boldsymbol{\theta}}(\tilde{\boldsymbol{x}})^{T}U\boldsymbol{b}}{\boldsymbol{b}^{T}\Lambda\boldsymbol{b}}\right]\label{eq:B^ub_part}
\end{multline}
Now, write $\boldsymbol{b}=\Lambda^{-\frac{1}{2}}\hat{\boldsymbol{b}}+\boldsymbol{b}_{\perp}$, where   $\boldsymbol{b}^T_{\perp} = [\boldsymbol{0}_k, \tilde{\boldsymbol{b}}]$  such that $\tilde{\boldsymbol{b}}\in\mathbb{R}^{(P-k)\times 1}$. Note that, $\Lambda^{\frac{1}{2}}\boldsymbol{b}=\hat{\boldsymbol{b}}$. Partition  $U = [U_k, U_{\perp}]$, with $U_k\in \mathrm{R}^{k\times P}$.
The numerator in (\ref{eq:B^ub_part}) is then
\begin{multline}
\boldsymbol{b}^{T}U^{T}J_{\boldsymbol{\theta}}(\tilde{\boldsymbol{x}})J_{\boldsymbol{\theta}}(\tilde{\boldsymbol{x}})^{T}U\boldsymbol{b} = \hat{\boldsymbol{b}}^{T}\Lambda^{-\frac{1}{2}}U^{T}J_{\boldsymbol{\theta}}(\tilde{\boldsymbol{x}})J_{\boldsymbol{\theta}}(\tilde{\boldsymbol{x}})^{T}U\Lambda^{-\frac{1}{2}}\hat{\boldsymbol{b}}\\
+2\tilde{\boldsymbol{b}}^T U_{\perp}^{T}J_{\boldsymbol{\theta}}(\tilde{\boldsymbol{x}})J_{\boldsymbol{\theta}}(\tilde{\boldsymbol{x}})^{T}U\Lambda^{-\frac{1}{2}}\hat{\boldsymbol{b}}
+\tilde{\boldsymbol{b}}^T U_{\perp}^{T}J_{\boldsymbol{\theta}}(\tilde{\boldsymbol{x}})J_{\boldsymbol{\theta}}(\tilde{\boldsymbol{x}})^{T}U_{\perp}\tilde{\boldsymbol{b}}\\
=\hat{\boldsymbol{b}}^{T}\Lambda^{-\frac{1}{2}}U^{T}J_{\boldsymbol{\theta}}(\tilde{\boldsymbol{x}})J_{\boldsymbol{\theta}}(\tilde{\boldsymbol{x}})^{T}U\Lambda^{-\frac{1}{2}}\hat{\boldsymbol{b}}, \label{eq:numer}
\end{multline}
where the last equality is because the columns  of $J_{\boldsymbol{\theta}}(\tilde{\boldsymbol{x}})$ almost surely belong to the span of
$A(\boldsymbol{\theta})=\mathbb{E}_{\boldsymbol{x}} J_{\boldsymbol{\theta}}(\tilde{\boldsymbol{x}})J_{\boldsymbol{\theta}}(\tilde{\boldsymbol{x}})^{T}$, and in particular
are orthogonal to the rows of $U_{\perp}^T$. Substituting  (\ref{eq:numer}) back in  (\ref{eq:B^ub_part}) we obtain that
\begin{eqnarray*}
	&&B^{\mathrm{ub}}=\frac{\sigma_{\varepsilon}^{2}}{M}\mathbb{E}_{\boldsymbol{\theta},\tilde{\boldsymbol{x}}}\left[\max_{\hat{\boldsymbol{b}}\in\mathbb{R}^{P}}\frac{\hat{\boldsymbol{b}}^{T}\Lambda^{-\frac{1}{2}}U^{T}J_{\boldsymbol{\theta}}(\tilde{\boldsymbol{x}})J_{\boldsymbol{\theta}}(\tilde{\boldsymbol{x}})^{T}U\Lambda^{-\frac{1}{2}}\hat{\boldsymbol{b}}}{\left\Vert \hat{\boldsymbol{b}}\right\Vert _{2}^{2}}\right]
	\\&=&\frac{\sigma_{\varepsilon}^{2}}{M}\mathbb{E}_{\boldsymbol{\theta},\tilde{\boldsymbol{x}}}\left[\mathrm{Tr}\left(\Lambda^{-\frac{1}{2}}U^{T}J_{\boldsymbol{\theta}}(\tilde{\boldsymbol{x}})J_{\boldsymbol{\theta}}(\tilde{\boldsymbol{x}})^{T}U\Lambda^{-\frac{1}{2}}\right)\right]
	=\frac{\sigma_{\varepsilon}^{2}}{M}\mathbb{E}_{\boldsymbol{\theta}}\left[\mathrm{Tr}\left(AU\Lambda^{-\frac{1}{2}}\Lambda^{-\frac{1}{2}}U^{T}\right)\right]\\
	&=&\frac{\sigma_{\varepsilon}^{2}}{M}\mathbb{E}_{\boldsymbol{\theta}}\left[\mathrm{Tr}\left(U\begin{bmatrix}I_k&0\\0&0\end{bmatrix}U^{T}\right)\right]=\sigma_{\varepsilon}^{2}\frac{\mathbb{E}_{\boldsymbol{\theta}}\left[k\right]}{M}.
\end{eqnarray*}
(In the last display, we used that the matrix $\Lambda^{-\frac{1}{2}}U^{T}J_{\boldsymbol{\theta}}(\tilde{\boldsymbol{x}})J_{\boldsymbol{\theta}}(\tilde{\boldsymbol{x}})^{T}U\Lambda^{-\frac{1}{2}}$
is of rank one in the first equality; in the next to last, we used that  $A=U\Lambda U^T$.) \qed

\subsection{An auxiliary lemma}
\begin{lemma}
	\label{lem: Gaussian expenssion}
	Let $(q_1,q_2)$ be a  bi-variate centered Gaussian vector, with covariance matrix $C=\begin{bmatrix}
	v_1^{2} & \varepsilon\\
	\varepsilon & v_2^{2}\end{bmatrix}$,
	such that $\eps<v_i^2$, $i=1,2$, and $v_i> 1/2$. Given   functions $f_i:\mathbb{R}\rightarrow \mathbb{R}$ with bounded  derivatives up to fifth order, we have that
	\begin{eqnarray}
	\mathbb{E}\left[f_1(q_1)f_2(q_2)\right]&=&\mathbb{E}\left[f_1(v_1z)\right]\mathbb{E}\left[f_2(v_2z)\right]\\
	&&+\varepsilon \mathbb{E}\left[f_1'(v_1z)\right]\mathbb{E}\left[f_2'(v_2z)\right]+\eps^2 M_{f_1,f_2}
	+O\big(\varepsilon^{5/2}\big), \nonumber
	\end{eqnarray}	
	where $z\sim\mathcal{N}(0,1)$, $M_{f_1,f_2}=\sum_{\ell=1}^L a_{\ell,f_1}(v_1) b_{\ell,f_2}(v_2)$ with $L<20$, and
	the constants $a_{\ell,f_i}(v_1), b_{\ell,f_i}(v_2)$ as well as the implied constant in the $O$ term are bounded by an absolute constant independent of $v_i,\eps$, and
	the   constants $a_{\ell,f_2}$ and $b_{\ell,f_2}$ are continuous  in $v_i$.
\end{lemma}
\noindent (The constant $L<20$ is just a finite constant which we do not bother to make explicit.)
\begin{proof}
	Using the assumption $\varepsilon<v_i^2$, we can represent $q_i$ as
	\[ q_i=\sqrt{(v_1^2-\eps)} G_i+\sqrt{\eps}G=
	v_i G_i+Q_i,\]
	with
	\[Q_i= \sqrt{\eps}G-c_{\eps,i}G_i, \quad c_{\eps,i}=v_i(1-\sqrt{v_i^2-\eps})=\frac{\eps}{2v_i}+O\big(\frac{\eps^2}{v_i^3}\big),\]
	where $G_1,G_2,G$ are iid standard centered Gaussians.  Using a Taylor expansion to fifth order , we write
	\[ f_i(q_i)=f_i(v_i G_i)+Q_i f_i'(v_i G_i)+
	\ldots+\frac{Q_i^4}{4!} f_i^{(4)}(v_i G_i)+ \frac{Q_i^5}{5!} f_i^{(5)}(\zeta_i),\]
	with $|\zeta_i-v_i G_i|\leq Q_i$. Note that $\E|Q_i|^5=O(\eps^{5/2})$.
	
	By Gaussian integration by parts,  $\E( G_i g(v_i G_i))=v_i g'(v_i G_i)$ for continuously differentiable functions $g$ that are bounded with their derivative.
	Using the fact that $c_{\eps,i}=\eps/2v_i+O(\eps^2/v_i^3)=
	\eps/2v_i+O(\eps^3)$, and the Taylor expansion, and writing
	$\tilde f_i=f_i(v_i G_i)$, $\tilde f_i'=f_i'(G_i v_i)$, etc.,  we obtain that
	\begin{equation}
	\label{eq-full}
	\E f_1(q_1) f_2(q_2)= \E \tilde f_1\E\tilde f_2 +\eps \E \tilde f_1'\E\tilde f_2'+ \eps^2 M_{f_1,f_2} +O(\eps^{5/2}),
	\end{equation}
	where $M_{f_1,f_2}=\sum_{\ell=1}^L a_{\ell,f_1}(v_1) b_{\ell,f_2}(v_2)$, and $L<20$ is an absolute constant independent of $v_i,\eps$, and due to the Taylor expansion,
	the constants $a_{\ell,f_2}$ and $b_{\ell,f_2}$ are continuous in $v_i$ by our smoothness assumption on $f_i$.  This completes the proof of the lemma.
\end{proof}

\appendix

\section{Background: the  Cram\'{e}r-Rao lower bounds \label{Sec:CRB}}

In this section, we review the Cram\'er-Rao bounds. We present first a lower bound on the smallest achievable
total variance of an
estimator. This bound is based on a variation of the basic Cram\'er-Rao (CR) bound  \cite{cramer1946contribution,rao1947minimum}. Here, it is presented in the presence of nuisance parameters \cite{miller1978modified}.
\begin{theorem}(\cite[Theorem 6.6]{lehmann2006theory})
	\label{thm:CRLB}Let $\hat{y}_{\boldsymbol{\theta}}(\tilde{\boldsymbol{x}},Z)$
	be a square integrable estimator. Suppose that
	\begin{enumerate}
		\item $\psi_{\boldsymbol{\theta}}(\tilde{\boldsymbol{x}})=\mathbb{E}_{Z|\boldsymbol{\theta},\tilde{\boldsymbol{x}}}\hat{y}$ and its derivative exist.
		\item $\frac{\partial p(Z|\boldsymbol{\theta})}{\partial\boldsymbol{\theta}_{i}}$
		exists and square integrable for all $i$.
	\end{enumerate}
	Define the likelihood function $l(Z|\boldsymbol{\theta})=\log p(Z|\boldsymbol{\theta})$. Then,  $\mathbb{E}_{Z|\boldsymbol{\theta}}\left[\frac{\partial l(Z|\boldsymbol{\theta})}{\partial\boldsymbol{\theta}_{i}}\right]=0$, and
	\begin{equation}
	\mathbb{E}_{\tilde{\boldsymbol{x}},\boldsymbol{\theta}}\mathrm{Var}_{Z|\boldsymbol{\theta},\tilde{\boldsymbol{x}}} \left[\hat{y}\right]\geq \mathbb{E}_{\boldsymbol{\theta},\tilde{\boldsymbol{x}}}\left[\max_{\boldsymbol{a}\in\mathbb{R}^{P}}\frac{\boldsymbol{a}^{T}\nabla_{\boldsymbol{\theta}}\psi_{\boldsymbol{\theta}}(\tilde{\boldsymbol{x}})\left(\nabla_{\boldsymbol{\theta}}\psi_{\boldsymbol{\theta}}(\tilde{\boldsymbol{x}})\right)^{T}\boldsymbol{a}}{\boldsymbol{a}^{T}I(\boldsymbol{\theta})\boldsymbol{a}}\right],\label{eq:CRBub}
	\end{equation}	
	where the $P\times P$ Fisher information matrix $I(\boldsymbol{\theta})$ is defined by
	\begin{equation}
	I_{ij}(\boldsymbol{\theta})=\mathbb{E}_{Z|\boldsymbol{\theta}}\left[\frac{\partial l(Z|\boldsymbol{\theta})}{\partial\boldsymbol{\theta}_{i}}\frac{\partial l(Z|\boldsymbol{\theta})}{\partial\boldsymbol{\theta}_{j}}\right],\quad i,j=1,\ldots,P.
	\label{eq:Fisher matrix}
	\end{equation}
\end{theorem}
\begin{remark}
	The formulation of  Theorem 6.6 in \cite{lehmann2006theory} assumes that the Fisher information  $I$ is  invertible. The proof of the theorem is based on Theorem 6.1 there. To obtain the version quoted above, one extends  Theorem 6.1 by using the display below (6.4) without the right most equality, and then repeats the proof of Theorem 6.6. Note that when $I$ is invertible,
Theorem \ref{thm:CRLB} is a reformulation of \cite[Theorem 6.6]{lehmann2006theory}.
\end{remark}

Using the Van
Trees (posterior) \cite{van2004detection,gill1995applications,gassiat2013revisiting} version of the Cram\'er-Rao inequality for multidimensional parameter space, one can derive a lower bound on the generalization error of a possibly biased  estimators $\hat{y}_{\boldsymbol{\theta}}(\tilde{\boldsymbol{x}},Z)$. We provide here a version of this inequality,
where we condition over part of the parameters' space:

\begin{theorem}(\cite[Theorem 1]{gill1995applications})
	Let $\hat{y}_{\boldsymbol{\theta}}(\tilde{\boldsymbol{x}},Z)$
	be an estimator of  $f_{\boldsymbol{\theta}}(\tilde{\boldsymbol{x}})$. Partition
 $\boldsymbol{\theta}=(\boldsymbol{\theta}_c,\boldsymbol{\theta}_l)\in\mathbb{R}^{P}$, where $\boldsymbol{\theta}_c\in \mathbb{R}^{N_c}$, and $\boldsymbol{\theta}_l\in \mathbb{R}^{N_l}$, such that $P=N_c+N_l$.
Suppose that,
	\begin{enumerate}
		\item The functions $f_{\boldsymbol{\theta}}(\tilde{\boldsymbol{x}})$,  $\nabla_{\boldsymbol{\theta}_c}f_{\boldsymbol{\theta}}(\tilde{\boldsymbol{x}})$ are absolutely continuous in $\boldsymbol{\theta}_{c}$ for almost all values of $\boldsymbol{\theta}_c$.
		\item $p(\boldsymbol{y},\boldsymbol{\theta}_{c}|\boldsymbol{\theta}_{l},X)$ is  Gaussian. (As in  (\ref{eq:model}).)
		\item $\boldsymbol{\theta}_{c}\sim \mathcal{N}(0,\alpha^{-1}I_{N_c})$ is independent of $\boldsymbol{\theta}_{l}$
		\item The Fisher information, $I_{\alpha}(\boldsymbol{\theta}_{l},X)\in\mathbb{R}^{N_c\times N_c}$ exists and $diag(I_{\alpha}(\boldsymbol{\theta},X))^{1/2}$ is locally integrable in $\boldsymbol{\theta}_{c}$, where
		\begin{equation}
		I_{\alpha}(\boldsymbol{\theta}_{l},X)=\sum_{k=1}^{M}\mathbb{E}_{\boldsymbol{\theta}_{c}|\boldsymbol{\theta}_{l},X}\left[\nabla_{\boldsymbol{\theta}_{c}} f_{\boldsymbol{\theta}}(\boldsymbol{x}^{(k)})\left(\nabla_{\boldsymbol{\theta}_{c}} f_{\boldsymbol{\theta}}(\boldsymbol{x}^{(k)}) \right)^{T}\right]+\alpha I_{N_c}.\label{eq:conditional Fisher_model1}
		\end{equation}
	\end{enumerate}
	Then, with $J(\boldsymbol{\theta}_l,\tilde{\boldsymbol{x}})=\mathbb{E}_{{\boldsymbol{\theta}_{c}}|\boldsymbol{\theta}_{l},\tilde{\boldsymbol{x}}}\nabla_{\boldsymbol{\theta}_{c}}f_{\boldsymbol{\theta}}(\tilde{\boldsymbol{x}})$, we have that
	\begin{equation}
	\mathbb{E}(\hat y-f_{\boldsymbol{\theta}}(\tilde{\boldsymbol{x}}))^2\geq  \mathbb{E}_{\boldsymbol{\theta}_{l},\tilde{\boldsymbol{x}},X}\left[ J(\boldsymbol{\theta}_l,\tilde{\boldsymbol{x}})^{T}\left(I_{\alpha}(\boldsymbol{\theta}_{l},X)\right)^{-1}J(\boldsymbol{\theta}_l,\tilde{\boldsymbol{x}})\right].
	\label{eq:vanTreesBoundGeneral}
	\end{equation}
	\label{Thm: conditional bound}
\end{theorem}
\begin{remark}
\cite[Theorem 1]{gill1995applications} is presented in more general settings and for more general distribution $p(Y,\boldsymbol{\theta}_{c}|\boldsymbol{\theta}_{l},X)$. Theorem \ref{Thm: conditional bound} is derived from
	\cite[Theorem 1]{gill1995applications} by taking $C=J(\boldsymbol{\theta}_l,\tilde{\boldsymbol{x}})^{T}\left(I_{\alpha}(\boldsymbol{\theta}_{l},X)\right)^{-1}\in\mathbb{R}^{1\times N_c}$ independent on $\boldsymbol{\theta}_c$ and $B=1$. Also, in our setting, we condition on the features, $X$, and part of the parameters space, $\boldsymbol{\theta}_l$, and take the expectation only over the measurement  noise and $\boldsymbol{\theta}_c$.
We further note that
	\cite[Theorem 1]{gill1995applications} requires $\boldsymbol{\theta}$ to be in compact support. This assumption can be relaxed to include the Gaussian case,
 see for example \cite[Theorem 1]{gassiat2013revisiting}.
\end{remark}
\begin{remark}
	Different partitions of the parameters' space may provide different bounds, the choice of the best bound depends on the specific structure of the problem \cite{bobrovsky1987some}.
\end{remark}

Using Theorem \ref{Thm: conditional bound}, we provide the following bound which utilizes the layered structure of feed-forward networks.
\begin{corollary}
	\label{Cor: Layer-wise bound} Setup as in Theorem \ref{Thm: conditional bound}. Let  $f_{\boldsymbol{\theta}}(\tilde{\boldsymbol{x}})$ be as in
(\ref{eq: deep model}) and let Assumption \ref{subsec:assumptionsiFF1} hold.
 Then,
	\begin{equation}
	E_{g}\geq \max_{l\in[1,L]} B^{(l)},
	\end{equation}
with
	\begin{equation}
	B^{(l)}: =\mathbb{E}_{\boldsymbol{\theta}_{l},\tilde{\boldsymbol{x}},X}\left[J(\boldsymbol{\theta}_l,\tilde{\boldsymbol{x}})^{T}\left(I_{\alpha}^{{(l)}}(\boldsymbol{\theta}_{l},X)\right)^{-1}J(\boldsymbol{\theta}_l,\tilde{\boldsymbol{x}})\right],\label{eq:B^{(l)}}
	\end{equation}
	where  $J(\boldsymbol{\theta}_l,\tilde{\boldsymbol{x}})=\left(\mathbb{E}_{{W^{(l)}}|\boldsymbol{\theta}_{l},\tilde{\boldsymbol{x}}}\nabla_{W^{(l)}} f_{\boldsymbol{\theta}}(\tilde{\boldsymbol{x}})\right)$,
$\boldsymbol{\theta}_{l}$ are the parameters of the network without the $l$th layer $W^{(l)}$, and
the   conditional Fisher matrix $I_{\alpha}^{{(l)}}(\boldsymbol{\theta}_{l},X)\in\mathbb{R}^{N_{l}N_{l-1}\times N_{l}N_{l-1}}$ is
	\begin{equation}
	I_{\alpha}^{{(l)}}(\boldsymbol{\theta}_{l},X)=\sum^{M}_{k=1}\mathbb{E}_{{W}^{(l)}|\boldsymbol{\theta}_{l},X}\left[\nabla_{W^{(l)}} f_{\boldsymbol{\theta}}(\boldsymbol{x}_{k})\left(\nabla_{W^{(l)}} f_{\boldsymbol{\theta}}(\boldsymbol{x}_{k})\right)^{T}\right]+\alpha_lI_{N_lN_{l-1}},\label{eq:reduced Fisher layered model}.
	\end{equation}
\end{corollary}
\begin{proof}
	Apply Theorem \ref{Thm: conditional bound}  L times: for a given $l$, choose $\boldsymbol{\theta}_{c}=W^{(l)}$, and $\boldsymbol{\theta}_{l}$ to be the remaining parameters of the model, and then take the best over all bounds. Substitution in  (\ref{eq:conditional Fisher_model1}) yields  (\ref{eq:reduced Fisher layered model}).
\end{proof}

\section{Unbiased estimator - examples \label{Sec:rankI}}
%In this section, we derive a lower bound for the generalization error of unbiased estimators.
%if we calculate the Fisher information $I(\boldsymbol{\bm{\theta}})\in \mathbb{R}^{P\times P}$ for the distribution $p(Z|\theta)$, we obtain
%\begin{equation}
%I(\boldsymbol{\theta}) = \frac{M}{\sigma_{\varepsilon}^2}A(\boldsymbol{\theta}),\quad
%A(\boldsymbol{\theta})=\mathbb{E}_{\boldsymbol{x}}\left[\nabla_{\boldsymbol{\theta}} f_{\boldsymbol{\theta}}(\boldsymbol{x})\nabla_{\boldsymbol{\theta}} f_{\boldsymbol{\theta}}(\boldsymbol{x})^T\right].	\label{eq:Fisher_model}
%\end{equation}
%\begin{theorem}
%	\label{Thm:CRLB_unbiased} Consider the model in \eqref{eq:model}.
%	Assume that $\nabla_{\boldsymbol{\theta}} f_{\boldsymbol{\theta}}(\boldsymbol{x})$ exists and is square integrable under
%	$p(\boldsymbol{x})$. Let $\hat{y}(\tilde{\boldsymbol{x}},Z)$ be any square integrable unbiased estimator. Then,
%	\begin{equation}
%	\mathbb{E}\left(\tilde{y}-\hat{y}\right)^{2}\geq \sigma_{\varepsilon}^{2}\frac{\mathbb{E}_{\boldsymbol{\theta}}\left[\mathrm{Rank}\left(I(\boldsymbol{\theta})\right)\right]}{M}.
%	\label{eq:CRLB_unbiased}
%	\end{equation}
%\end{theorem}
%The proof of Theorem \ref{Thm:CRLB_unbiased} is provided in Section \ref{Sec:CRLB_unbiased_proof}.
%\begin{remark}
%The expected rank of the Fisher information matrix can be viewed as the ``interpolation threshold'' for non-linear models.
%\end{remark}
In this section, we specialize Theorem \ref{Thm:CRLB_unbiased} to various network architectures. Auxiliary results
concerning the structure of the Fisher information matrix are collected in Section \ref{Sec:FisherMatrixStructure}.
As we will see, nonlinear models exhibit a much higher rank of the Fisher information matrix.  
%In our analysis, we will need a few
\iffalse
\subsection{Linear regression \label{subsubsec: lR unbiased }}
\begin{corollary} \label{Corr:lR_unbiased}
	Consider the model in \eqref{eq:model}, with  $f_{\boldsymbol{\theta}}(\boldsymbol{x})=\boldsymbol{\theta}^{T}\boldsymbol{x}/\sqrt{d}$.
	Assume that $p(\boldsymbol{\theta})$ has finite covariance matrix.
	Let $\hat{y}(\tilde{\boldsymbol{x}},Z)$ be any unbiased square integrable estimator.
	Then,
	\begin{equation}
	\mathbb{E}\left(\tilde{y}-\hat{y}\right)^{2}\geq \frac{\sigma_{\varepsilon}^{2}}{M}\,\mathrm{Rank\left({\Sigma}\right)}, \quad
	\Sigma=\mathbb{E}\left[\boldsymbol{x}\boldsymbol{x}^T\right].
	\label{eq:CRLB_unbiased-Reg}
	\end{equation}
\end{corollary}
Note that the model of Corollary \ref{Corr:lR_unbiased} corresponds to \eqref{eq: deep model} with $L=1$ and $\sigma(x)=x$.
\begin{proof}
	We apply  Theorem \ref{Thm:CRLB_unbiased}. The Fisher information matrix in  (\ref{eq:Fisher_model}) reads
	$
	I(\boldsymbol{\theta})=M\mathbb{E}\left[\boldsymbol{x}\boldsymbol{x}^{T}\right]/(\sigma_{\varepsilon}^{2}d)=M\Sigma/(\sigma_{\varepsilon}^{2}d)
	$. Substituting in (\ref{eq:CRLB_unbiased})  complete the proof.
\end{proof}
\fi
\subsection{Linear activation function \label{subsubsec: linear net}}
\begin{corollary}
	\label{Corr:depp_linear_unbiased}  Consider the model in \eqref{eq:model}, (\ref{eq: deep model}) with $\sigma(x)=x$, $N_L=1$ and $L\ge1$ layers. Assume that $p(\boldsymbol{\theta})$ has finite covariance matrix. Let $\hat{y}(\tilde{\boldsymbol{x}},Z)$ be any unbiased square integrable estimator. Then
	\begin{equation}
	\mathbb{E}\left(\tilde{y}-\hat{y}\right)^{2}\geq \frac{\sigma_{\varepsilon}^{2}}{M}\,\mathrm{Rank\left({\Sigma}\right)}, \quad
	\Sigma=\mathbb{E}\left[\boldsymbol{x}\boldsymbol{x}^T\right].
	\label{eq:CRLB_unbiased-Reg}
	\end{equation}
	%\eqref{eq:CRLB_unbiased-Reg} holds.
\end{corollary}

\begin{proof}
We again apply theorem \ref{Thm:CRLB_unbiased}. By (\ref{eq:CRLB_unbiased}), we need to calculate the rank of the Fisher matrix $I(\boldsymbol{\theta})$, see
(\ref{eq:Fisher_model}), for $f_{\boldsymbol{\theta}}(\boldsymbol{x})$ defined in(\ref{eq: deep model}) with $\sigma(x)=x$. By Lemma \ref{lem:Linear-Fisher-matrix} below,
the Fisher matrix can be written as $I(\boldsymbol{\theta})=\frac {M}{\sigma_{\varepsilon}^2} \left(\prod_{l=1}^{L}\frac{1}{N_{l-1}}\right)J_LJ_L^T$, where the matrix, $J_{L}\in \mathbb{R}^{P\times dN_L}$ is defined in (\ref{eq:I_lin}). Taking $N_L=1$, we have that,   $\mathbb{E}_{\boldsymbol{\theta}}\left[\mathrm{Rank}(I(\boldsymbol{\theta}))\right]=\mathrm{Rank}(I(\boldsymbol{\theta}))=\mathrm{Rank}(J_L)=d$, since $P>d$, which completes the proof.
\end{proof}
\subsection{Nonlinear activation function \label{subsubsec: nonlinear net unbiased}}

Matrix rank is unstable with respect to perturbations. As we now show,
even weak non-linearities generate small eigenvalues in the Fisher matrix, and increases dramatically its rank.
\begin{theorem}\label{Thm:deep_nonlinear_unbiased}
	Consider the model
	in \eqref{eq:model}, (\ref{eq: deep model}) with $L=2$, and let Assumptions \ref{subsec:assumptions},  \ref{subsec:assumptionsiFF1} and \ref{subsec:assumptionsiFF2}
	hold, with activation function $\sigma$  satisfying $\eta_1>\xi$, i.e., the function $\sigma'$ is not identically constant.
	In the regime $N_1,d\rightarrow \infty$ such that $\beta_{1}=\lim_{d\to\infty}{N_1}/d\in (0,\infty)$, we have that
	$\mathbb{E}_{\boldsymbol{\theta}}\left[\mathrm{Rank}(I(\boldsymbol{\theta}))\right]=\beta_1 d^2(1+o(1))$,
	where  $I(\boldsymbol{\theta})$  is as in  (\ref{eq:Fisher_model}).	
\end{theorem}

\begin{remark}
	The same proof, using the recursive structure of the function $f_{\bm{\theta}}$,
	shows that the conclusion of Theorem \ref{Thm:deep_nonlinear_unbiased} hold for any $L\geq 2$.
	%Due to the recursive structure of the function $f_{\bm{\theta}}$.
\end{remark}
\begin{remark}
  \label{rem:B.4}
	It follows in particular that in the asymptotic regime described in Theorem \ref{Thm:deep_nonlinear_unbiased}, the generalization error of any unbiased estimator
	is worse than the Bayesian error. While somewhat counter-intuitive, it is consistent with the fact that in general, biased estimators can
	sometimes achieve a lower error than the CR bound
	for unbiased estimator. For example, a better bound can be achieved
	for biased estimators  by taking the lowest bias, see \cite{ben2009lower}.
\end{remark}
\begin{proof}
We work in the regime where
$\beta_1=\lim_{d\to\infty} N_1/d\in(0,\infty)$.
Recall the definition  of
$f_{\boldsymbol{\theta}}(\boldsymbol{x})=(\boldsymbol{w}^{(2)})^{T}\sigma(\boldsymbol{q})/\sqrt{N}$, where $\boldsymbol{q}=W^{(1)}\boldsymbol{x}/\sqrt{d}$,
see (\ref{eq: deep model}) with $L=2$, $N_1=N$  and $N_2=1$ of gradient
\begin{equation}
J_{\boldsymbol{\theta}}(\tilde{\boldsymbol{x}}):=\nabla_{\boldsymbol{\theta}}f_{\boldsymbol{\theta}}(\tilde{\boldsymbol{x}})
=
\begin{bmatrix}
\nabla_{W^{(1)}}f_{\boldsymbol{\theta}}(\boldsymbol{x})\\
\nabla_{\boldsymbol{w}^{(2)}}f_{\boldsymbol{\theta}}(\boldsymbol{x})
\end{bmatrix}
=\begin{bmatrix}	\frac{1}{\sqrt{d}}s_w^{1}\otimes\boldsymbol{x}\\
\frac{1}{\sqrt{N}}\sigma(\boldsymbol{q})
\end{bmatrix}
\end{equation}
where
$s_w^{1}=	D^{1}\boldsymbol{w}^{(2)}/\sqrt{N}\in\mathbb{R}^{N}$, with
$D^{1}=\mathrm{diag}\left(\sigma'(\boldsymbol{q})\right)$.
The Fisher information matrix for the model in  (\ref{eq:model}) as defined in  (\ref{eq:Fisher_model}) has the following form:
\begin{eqnarray}\label{eq:Block Fisher 2 layers}
A&=&\mathbb{E}_{\boldsymbol{x}}\left[J_{\boldsymbol{\theta}}(\boldsymbol{x})J_{\boldsymbol{\theta}}(\boldsymbol{x})^{T}\right]\\
&=&\frac{1}{N}
\begin{bmatrix}
\beta_1\mathbb{E}_{\boldsymbol{x}}\left[(s_w^{1}\otimes\boldsymbol{x})(\boldsymbol{x}^{T}\otimes s_w^{1T})\right] & \sqrt{\beta_1}\mathbb{E}_{\boldsymbol{x}}\left[(s_w^{1}\otimes\boldsymbol{x})\sigma(\boldsymbol{q})^T\right]\\
\sqrt{\beta_1}\mathbb{E}_{\boldsymbol{x}}\left[\sigma(\boldsymbol{q})(\boldsymbol{x}^{T}\otimes s_w^{1T})\right] & \mathbb{E}_{\boldsymbol{x}}\left[\sigma(\boldsymbol{q})\sigma(\boldsymbol{q})^{T}\right]
\end{bmatrix}\nonumber\\
&=:&\begin{bmatrix}
A^{(1)}\in\mathbb{R}^{Nd\times Nd} & A^{(2)}\in\mathbb{R}^{Nd\times N}\\
A^{(2)T}\in\mathbb{R}^{N\times Nd} & A^{(3)}\in\mathbb{R}^{N\times N}
\end{bmatrix}
\nonumber\end{eqnarray}
For the lemma, we only care about $A^{(1)}=\frac{1}{dN}A_R$ the top $Nd\times Nd$ block of
$A(\btheta)$, where
\begin{equation}
\label{eq-AR}
A_R=\mathbb{E}_{\boldsymbol{x}}\left[D^{1}\boldsymbol{w}^{(2)}(\boldsymbol{w}^{(2)})^TD^{1}\otimes\boldsymbol{x}\boldsymbol{x}^{T}\right].
\end{equation}
Using Lemma \ref{lem:Asymptotics of the Fisher}, the  matrix $A_R(\boldsymbol{\theta})$ can be rewritten as
\begin{equation}
{A}_R=D^{(2)}\otimes I_d +B+\mathcal{E}_w
\end{equation}
such that $\mathrm{Rank}(B)=o(Nd)$, and $\|\mathcal{E}_w\|_{\mathrm{HS}}=o(Nd)$ in probability. The matrix $D^{(2)}\in \mathbb{R}^{N\times N}$ is a diagonal matrix, where the $(i,j)$ element is $D^{(2)}_{ij}=\delta_{ij}\sigma_{x}^{2}a_i(w_i^{(2)})^2$. The constants $a_i=\eta_{1,i}-\theta_{1,i}>0$ by assumption for all $i\in[1,N_1]$ and are independent of the vector $\bm{w}^{(2)}$. Using spectral decomposition, the matrix $\mathcal{E}_w$ can  be decomposed as $\mathcal{E}_w=\mathcal{E}_{w1}+\mathcal{E}_{w2}$,  such that $\mathrm{Rank}(\mathcal{E}_{w1})=o(N_1d)$, and $\|\mathcal{E}_{w2}\|_{\mathrm{op}}\rightarrow 0$.
Hence, we can  write
\begin{equation}
{A}_R=D^{(2)}\otimes I_d +\tilde{B}+\mathcal{E}_{w2}
\end{equation}
such that $\tilde{B}=B+\mathcal{E}_{w1}$, where $\mathrm{Rank}(\tilde{B})=o(Nd)$.
We will now bound the eigenvalues of $A_R(\btheta)$ from below to estimate its rank. Define the events $\mathcal{A}_\varepsilon=\{\| \mathcal{E}_{w2}\|_\mathrm{op}\leq\varepsilon\}$ and $\mathcal{B}_\delta=\{\#\{i\in 1,\ldots,N:
(a_i{w}_i^{(2)})^2<2\varepsilon \} >\delta N\}$. We know that $\mathbb{P}(\mathcal{A}_\varepsilon)\rightarrow 1$. We now note, using the Gaussian density of the $w_i^{(2)}$s,  that
\begin{equation}
\mathbb{P}(\mathcal{B}_{\delta})\leq\frac{N\mathbb{P}(a_i({w}_i^{(2)})^2<2\varepsilon)}{\delta N}\leq c_1\frac{\sqrt{\varepsilon}}{\delta}\label{eq:Pb}
\end{equation}
with some fixed constant $c_1>0$. Taking $\delta=\varepsilon^{1/4}$, we thus conclude that $\mathbb{P}(\mathcal{B}_{\varepsilon^{1/4}})\to 0$, and therefore
$\mathbb{P}(\mathcal{B}^c_{\varepsilon^{1/4}}\cap\mathcal{A}_\varepsilon)
\to 1$.
On the event $\mathcal{B}^c_{\varepsilon^{1/4}}\cap\mathcal{A}_\varepsilon$, we have
that $\#\{i:a_i({w}_i^{(2)})^2\geq2\varepsilon\}\geq (1-\varepsilon^{1/4})Nd$, and also
$\| {\mathcal{E}_{w2}}\|_\mathrm{op}\leq \varepsilon$. Thus, using Weyl's inequality and the fact that $\mathrm{Rank}(\tilde{B})=o(N_1d)$,
on this event
$\mathrm{Rank}(A_R(\btheta))\geq(1-2\varepsilon^{1/4})Nd$, for large $N$. The expected rank is then
bounded from below, for large $N$,  by
\begin{equation}
\mathbb{E}_{\boldsymbol{\theta}}\left[\mathrm{Rank}(A_R)\right]\geq(1-2\varepsilon^{1/4})Nd\, \mathbb{P}(\mathcal{B}^c_{\varepsilon^{1/4}}\cap\mathcal{A}_\varepsilon)\geq(1-2\varepsilon^{1/4})Nd(1-o(1)).
\end{equation}
We thus conclude that
$\mathbb{E}_{\boldsymbol{\theta}}\left[\mathrm{Rank}(A_R(\btheta))\right]=(1-o(1))\beta_1d^2$.
Now, note that the rank of $A_R(\btheta)$ bounds from below the rank of $A(\btheta)$, and therefore,
$\mathbb{E}_{\boldsymbol{\theta}}\left[\mathrm{Rank}(I(\boldsymbol{\theta}))\right]\geq
\beta_1 d^2(1-o(1))$, which completes the proof.
\end{proof}

\section{Structure of the Fisher information matrix of a feed-forward neural network  \label{Sec:FisherMatrixStructure}}
In this section, we analyze the structure of the Fisher information matrix of feed-forward  networks; Lemma \ref{lem:Linear-Fisher-matrix} considers
linear activation functions and  any number of layers, while Lemma \ref{lem:Asymptotics of the Fisher} considers two layered networks with
nonlinear activation functions.
This analysis  is then used in calculating the expected rank of the Fisher information in Corollary \ref{Corr:depp_linear_unbiased} and Theorem \ref{Thm:deep_nonlinear_unbiased}.

\begin{lemma}
	\label{lem:Linear-Fisher-matrix} Suppose $\mathbb{E}\boldsymbol{x}\boldsymbol{x}^T=\Sigma$, then the matrix defined in (\ref{eq:Fisher_model}) with $f_{\boldsymbol{\theta}}(\boldsymbol{x})$ as in (\ref{eq: deep model}) and $\sigma(x)=x$ can be decomposed as
	\begin{equation}
	I_1=\left(\prod_{l=1}^{L}\frac{1}{N_{l-1}}\right)J_{L}J_{L}^{T}
	\label{eq:I_lin}
	\end{equation}
	such that, the matrix  $J_{L}\in\mathbb{R}^{P\times dN_{L}}$ is composed of $L$ blocks where the $l$th block is $J_{L}^{(l)} =B_{l}^{T}\otimes A_{l}\Sigma^{1/2} \in \mathbb{R}^{N_{l}N_{l-1}\times dN_{L}}$, and
	\begin{equation}
	B_{l}^{T}=\begin{cases}
	\begin{array}{c}
	I_{N_{L}}\\
	W^{(L)}\\
	\Pi_{m=L}^{l+1}(W^{(m)})^T
	\end{array} & \begin{array}{c}
	l=L\\
	l=L-1\\
	1\leq l<L-1
	\end{array}\end{cases},\label{eq:B_l}
	\end{equation}
	and
	\begin{equation}
	A_{l}=\begin{cases}
	\begin{array}{c}
	I_{d}\\
	\Pi_{m=1}^{l-1}W^{(m)}
	\end{array} & \begin{array}{c}
	l=1\\
	1<l\leq L
	\end{array}\end{cases}.\label{eq:A_l}
	\end{equation}
\end{lemma}

\begin{remark}
	Similar results are  derived for $L=2$ in {\cite{pennington2018spectrum})}. Here, we consider general $L$.
\end{remark}
\begin{proof}
	We use  (\ref{eq:Fisher_model}) with the following model for a deep neural network with linear activation ( (\ref{eq: deep model})):
	\[
	f_{\boldsymbol{\theta}}(\boldsymbol{x})=\left(\prod_{l=1}^{L}\frac{1}{\sqrt{N_{l-1}}}W^{(l)}\right)\boldsymbol{x}
	\]
	The Fisher matrix is of size ${P\times P}$, where $P=\sum_{i=1}^{L}N_{i}N_{i-1}$ is the total number of parameters. It
	is composed of blocks,
	\begin{equation}
	I_1=\mathbb{E}_{\boldsymbol{x}}\left[\protect\begin{array}{cccc}
	J^{(1)}(J^{(1)})^{T} & J^{(1)}(J^{(2)})^{T} & \cdots & J^{(1)}(J^{(L)})^{T}\protect\\
	J^{(2)}(J^{(1)})^{T} & J^{(2)}(J^{(2)})^{T} & \cdots & J^{(2)}(J^{(L)})^{T}\protect\\
	\vdots & \vdots & \ddots & \vdots\protect\\
	J^{(L)}(J^{(1)})^{T} & J^{(L)}(J^{(2)})^{T} & \cdots & J^{(L)}(J^{(L)})^{T}
	\protect\end{array}\right],\label{eq:Fisher matrix L layers}
	\end{equation}
	where $J^{(l)}$, the Jacobean matrix of  $l$th layer is defined as follow, for $1\leq l<L$:
	\begin{equation}
	J^{(l)}=\nabla_{W^{(1)}} f_{\boldsymbol{\theta}}(\boldsymbol{x})=\Pi_{m=L}^{l+1}(W^{(m)})^T\otimes\Pi_{m=l-1}^{1}W^{(m)}\boldsymbol{x}\in\mathbb{R}^{N_{l-1}\times N_{l}}.
	\end{equation}
	We rewrite this matrix as follows:
	\[
	J^{(l)}=B_{l}^{T}\otimes A_{l}\boldsymbol{x}\in\mathbb{R}^{N_{l-1}N_{l}\times N_{L}}
	\]
	where $A_{l}\in\mathbb{R}^{N_{l-1}\times d}$, and $B_{l}\in\mathbb{R}^{N_{L}\times N_{l}}$, are defined in  (\ref{eq:A_l}) and  (\ref{eq:B_l}). The $(l,m)$th block
	of the Fisher matrix in  (\ref{eq:Fisher matrix L layers}):
	\begin{multline}
	\mathbb{E}_{\boldsymbol{x}}J^{(l)}(J^{(m)})^{T}=B_{l}^{T}B_{m}\otimes A_{l}\mathbb{E}\boldsymbol{x}\boldsymbol{x}^{T}A_{m}^{T}=\sigma_{x}^{2}B_{l}^{T}B_{m}\otimes A_{l}\Sigma A_{m}^{T}\\
	=\left(B_{l}^{T}\otimes A_{l} \Sigma^{1/2} \right)\left(B_{m}\otimes \Sigma^{1/2}A_{m}^{T}\right)=\sigma_{x}^{2}J_{L}^{(l)}(J_{L}^{(m)})^{T}\label{eq:block_l_m}
	\end{multline}
	Denoting by $J_{L}^{(l)}=B_{l}^{T}\otimes A_{l} \Sigma^{1/2} \in \mathbb{R}^{N_{l}N_{l-1}\times dN_{L}}$. Substituting  (\ref{eq:block_l_m}) for all blocks in  (\ref{eq:Fisher matrix L layers}) yield the desired result, i.e.,
	the Fisher matrix for linear activation function
	can be written as follows:
	\begin{equation}
	I_1=\left(\prod_{l=1}^{L}\frac{1}{N_{l-1}}\right)J_{L}J_{L}^{T}
	\label{eq:Linear_Fishe},
	\end{equation}
	such that the matrix  $J_{L}\in\mathbb{R}^{P\times dN_{L}}$ is composed of $L$ blocks where the $l$th block is $J_{L}^{(l)}$.
\end{proof}

\begin{lemma}
	\label{lem:Asymptotics of the Fisher}
	Let Assumptions \ref{subsec:assumptions}, \ref{subsec:assumptionsiFF1} and \ref{subsec:assumptionsiFF2} hold.
	Let $A_R$ be as in \eqref{eq-AR}. Then, with  $\beta_1=\lim_{d\to\infty} N/d\in (0,\infty)$, there exist
	matrices
	$D^{(2)},B$ and $\mathcal{E}_w$  such that
	\begin{equation}
	{A}_R=D^{(2)}\otimes I_d +B+\mathcal{E}_w,
	\end{equation}
	with
	\begin{enumerate}
		\item $\mathrm{Rank}(B)=o(d^2)$
		\item $\E(\|\mathcal{E}_w\|_{\mathrm{HS}}^2)=o(d^2)$
		\item $D^{(2)}\in \mathbb{R}^{N_1\times N_1}$ is a block diagonal matrix, where the $(i,j)$ element is $D^{(2)}_{ij}=\delta_{ij}\sigma_{x}^{2}(\eta_{1,i}-\theta_{1,i}^2)(w_i^{(2)})^2$, and
		$\theta_{1,i}= \mathbb{E}\left[\sigma'(v_{i}z)\right]$, $\eta_{1,i}= \mathbb{E}\left[\sigma'(v_{i}z)^2\right]$, $v^2_i=\sigma_x^2/d\sum_k (W_{ik}^{(1)})^2$,
		with
		$z\sim\mathcal{N}(0,1)$.
	\end{enumerate}
\end{lemma}

\begin{proof}[Proof of Lemma \ref{lem:Asymptotics of the Fisher}]
	Each element in the matrix $A_{R}$ is of the following form:
	\[
	A_{R,(i_1,i_2;i_1',i_2')}=w_{i_{2}}^{(2)}w_{i'_{2}}^{(2)}\mathbb{E}_{x}\left[\sigma'(q_{i_{2}})\sigma'(q_{i'_{2}})x_{i_{1}}x_{i'_{1}}\right].
	\]
	Set $\tilde{q}_{i'_{2}}=\frac{1}{\sqrt{d}}\sum_{l\neq i_1,i_1'}W_{i'_{2}l}^{(1)}x_{l}$,
	and $\tilde{q}_{i_{2}}=\frac{1}{\sqrt{d}}\sum_{l\neq i_1,i_1'}W_{i_{2}l}^{(1)}x_{l}$.
	Using a Taylor expansion of the function $\sigma'$ around these points up to third order,
	we have that for any $i_2, i_2'$ and $i_1=i_1'$,
	\begin{multline}
	\mathbb{E}_{x}\left[\sigma'(q_{i_{2}})\sigma'(q_{i'_{2}})x_{i_{1}}^2\right]= \sigma_{x}^{2}M^{(1)}_{i_2 i'_2}\\
	+\frac{6\sigma^4_x}{d}(M^{(2)}_{i_2 i'_2}W_{i'_2i_1}^{(1)}W_{i_2i_1}^{(1)}+M^{(13)}_{i_2 i'_2}(W_{i'_2i_1}^{(1)})^2
	+M^{(13)}_{i_2' i_2}(W_{i_2i_1}^{(1)})^2)+\mathcal{E}_{(i_2,i_2',i_1,i_1)}
	\label{eq:weak_correaltion_exp_diag}
	\end{multline}
	and for any $i_2, i_2'$ and $i_1\neq i_1'$,
	\begin{multline}
	\mathbb{E}_{x}\left[\sigma'(q_{i_{2}})\sigma'(q_{i'_{2}})x_{i_{1}}x_{i'_{1}}\right]=
	\frac{\sigma^4_x}{d}M^{(2)}_{i_2 i'_2}(W_{i'_2i_1}^{(1)}W_{i_2i'_1}^{(1)}+W_{i'_2i'_1}^{(1)}W_{i_2i_1}^{(1)})\\
	+\frac{\sigma_x^4}{d}(M_{i_2 i'_2}^{(13)}W_{i'_{2}i'_{1}}^{(1)}W_{i'_{2}i_{1}}^{(1)}+M_{i_2' i_2}^{(13)}W_{i_{2}i'_{1}}^{(1)}W_{i_{2}i_{1}}^{(1)})+\mathcal{E}_{(i_1,i_2,i_1',i_2')}
	\label{eq:weak_correaltion_exp}
	\end{multline}
	where we define $M^{(1)}_{i_2 i'_2}=\mathbb{E}_{x}\left[\sigma'(\tilde{q}_{i_{2}})\sigma'(\tilde{q}_{i'_{2}})\right]$, $M^{(2)}_{i_2 i'_2}=\mathbb{E}_{x}\left[\sigma''(\tilde{q}_{i_{2}})\sigma''(\tilde{q}_{i'_{2}})\right]$, and $M^{(13)}_{i_2 i'_2}=\mathbb{E}_{x}\left[\sigma'(\tilde{q}_{i_{2}})\sigma'''(\tilde{q}_{i'_{2}})\right]$. The matrix $\mathcal{E}$ is the remaining terms of order $d^{-3/2}$, i.e. $\E \mathcal{E}_{i_1,i_2,i_1',i_2'}^2\leq C/d^3$. To evaluate the "$M$" matrix elements we apply Lemma \ref{lem: Gaussian expenssion} with $\eps_{i_2 i_2'}=\sigma_x^2/d\sum_{k\neq i_1,i_1'} W_{i_2k}^{(1)}W_{i_2'k}^{(1)}$ and $v^2_{i}=\sigma_x^2/d\sum_{k\neq i_1,i_1'} (W_{ik}^{(1)})^2$ for $i=i_2, i_2'$. We will now expand the typical matrix elements in the four following cases; in each case, the matrix $\tilde{\mathcal{E}}$ satisfies conditions as noted:
	\begin{enumerate}
		\item $i_2\neq i_2'$ and  $i_1 = i_1'$
		\begin{multline}
		\mathbb{E}_{x}\left[\sigma'(q_{i_{2}})\sigma'(q_{i'_{2}})x_{i_{1}}^2\right]= \sigma_{x}^{2}(\theta_{1,i_2'}\theta_{1,i_2}
		+\theta_{2,i_2}\theta_{2,i_2'}\frac{\sigma_x^2}{d}\sum_{k\neq 1_1} W_{i_2k}^{(1)}W_{i_2'k}^{(1)}
		+O(\eps_{i_2 i_2'}^2))
		\\
		+\frac{6\sigma^4_x}{d}(\theta_{2,i_2'}\theta_{2,i_2}W_{i'_2i_1}^{(1)}W_{i_2i_1}^{(1)}+\theta_{3,i_2'}\theta_{1,i_2}(W_{i'_2i_1}^{(1)})^2
		+\theta_{3,i_2}\theta_{1,i_2'}(W_{i_2i_1}^{(1)})^2)+\mathcal{E}_{(i_1,i_2,i_1',i_2')}\\
		=\sigma_{x}^{2}\theta_{1,i_2'}\theta_{1,i_2}+
		\tilde{\mathcal{E}}_{(i_1,i_2,i_1',i_2')},
		\label{eq:weak_diag1}
		\end{multline}
		with $\E(\tilde{\mathcal{E}}_{(i_1,i_2,i_1',i_2')}^2)\leq C/d^2$.
		\item  $i_2\neq i_2'$ and $i_1\neq i_1'$
		\begin{multline}
		\mathbb{E}_{x}\left[\sigma'(q_{i_{2}})\sigma'(q_{i'_{2}})x_{i_{1}}x_{i'_{1}}\right]=
		\frac{\sigma^4_x}{d}\theta_{2,i_2}\theta_{2,i_2'}(W_{i'_2i_1}^{(1)}W_{i_2i'_1}^{(1)}+W_{i'_2i'_1}^{(1)}W_{i_2i_1}^{(1)})\\
		+\frac{\sigma_x^4}{d}(\theta_{3,i_2'}\theta_{1,i_2}W_{i'_{2}i'_{1}}^{(1)}W_{i'_{2}i_{1}}^{(1)}+\theta_{3,i_2}\theta_{1,i_2'}W_{i_{2}i'_{1}}^{(1)}W_{i_{2}i_{1}}^{(1)})+\tilde{\mathcal{E}}_{(i_1,i_2,i_1',i_2')},
		\label{eq:weak}
		\end{multline}
		where
		the  matrix element $\tilde{\mathcal{E}}_{(i_1,i_2,i_1',i_2')}$ is the matrix
		element  ${\mathcal{E}_{(i_1,i_2,i_1',i_2')}}$ plus higher order terms, whose second moment is bounded by $C/d^3$, and therefore $\E(\tilde{\mathcal{E}}_{(i_1,i_2,i_1',i_2')}^2)\leq C/d^3$.
		\item $i_2=i_2'$ and $i_1= i_1'$
		\begin{multline}
		\mathbb{E}_{x}\left[\sigma'(q_{i_{2}})^2x_{i_{1}}^2\right]= \sigma_{x}^{2}\eta_{1,i_2}
		+\frac{6\sigma^4_x}{d}(W_{i_2i_1}^{(1)})^2(\eta_{2,i_2}+2\eta_{31,i_2})
		+{\mathcal{E}}_{(i_1,i_2,i_1',i_2')}\\
		=\sigma_{x}^{2}\eta_{1,i_2}+ \tilde{\mathcal{E}}_{(i_1,i_2,i_1',i_2')},
		\label{eq:weak_diag1_diag2}
		\end{multline}
		with $\E(\tilde{\mathcal{E}}_{(i_1,i_2,i_1',i_2')}^2)\leq C/d^2$.
		\item $i_2=i_2'$ and $i_1\neq i_1'$
		\begin{multline}
		\E_{x}\left[\sigma'(q_{i_{2}})^2x_{i_{1}}x_{i'_{1}}\right]=
		\frac{2\sigma^4_x}{d}\eta_{2,i_2}W_{i_2i_1}^{(1)}W_{i_2i'_1}^{(1)}\\
		+\frac{2\sigma_x^4}{d}\eta_{31,i_2}W_{i_{2}i'_{1}}^{(1)}W_{i_{2}i_{1}}^{(1)}+{\mathcal{E}}_{(i_2,i_2',i_1,i_1')}= \tilde{\mathcal{E}}_{(i_2,i_2',i_1,i_1')},
		\label{eq:weak_diag2}
		\end{multline}
		with $\E(\tilde{\mathcal{E}}_{(i_1,i_2,i_1',i_2')}^2)\leq C/d^2$.
	\end{enumerate}
	Combining  (\ref{eq:weak_diag1}),  (\ref{eq:weak}),  (\ref{eq:weak_diag1_diag2}), and  (\ref{eq:weak_diag2}) and rearranging the elements in a matrix form we can now rewrite the matrix $A_R$ as
	\begin{equation}
	{A}_R=D^{(2)}\otimes I_d +B+\mathcal{E}_w,
	\end{equation}
	where $D^{(2)}$ is a block diagonal matrix of size $N_1$ whose  $(i,j)$th element is $D^{(2)}_{ij}=\delta_{ij}\sigma_{x}^{2}(\eta_{1,i}-\theta_{1,i}^2)(w_i^{(2)})^2$, and
	\begin{equation}
	B = J_wJ_w^T+\btheta_1\otimes A+\btheta_1^T\otimes A^T+R_wR_w^T+C_wC_w^T.
	\end{equation}
	The vector $J_w = (\btheta_2\circ\bm{w}^{(2)})\otimes I_d$, and $\theta_{k,i}$ is the $i$th element in the vector $\btheta_k\in \mathbb{R}^{N_1}$ for $k=1,2$. The matrix $A\in \mathbb{R}^{d \times dN_1}$ is a block matrix, composed of $N_1$ blocks of size $d\times d$. The $(i,j)$th element in the block $k$ is ${\sigma_x^4}\theta_{3,k}W_{ki}^{(1)}W_{kj}^{(1)}/{d}$. The vector $R_w,C_w\in\mathbb{R}^{N_1d}$ is composed of the rows and the columns of the matrix $\sigma_x^4W^{(1)}/d$,
	arranged in a vector Hadamard product with the vector $\btheta_2$, respectively. Therefore, $\mathrm{Rank}(B)\leq N+d+2$, since $\mathrm{Rank}(\btheta_1\otimes A)=\mathrm{Rank}(A)=d$. The matrix $\mathcal{E}_w$ has elements
	$\mathcal{E}_{w,i_2,i_2',i_1,i_1'}= {w}_{i_2}^{(2)}{w}_{i_2'}^{(2)}\tilde{\mathcal{E}}_{i_2,i_2',i_1,i_1'}$.
	Note that our estimates on the individual entries of $\mathcal{E}_w$ imply that
	$\E\left[\|\tilde{\mathcal{E}}_w\|_{\mathrm{HS}}^2\right]\leq Cd$ for some positive constant $C>0$. This completes the proof of the lemma.
\end{proof}

\small

\bibliographystyle{plain}

\end{document}